\documentclass[11pt]{article}\usepackage{graphicx} 
\usepackage{fullpage}

\usepackage{hyperref}
\usepackage[utf8]{inputenc} 
\usepackage[T1]{fontenc}    
\usepackage{url}            
\usepackage{mathtools}
\usepackage{booktabs}       
\usepackage{amsfonts}       
\usepackage{nicefrac}       
\usepackage{microtype}      
\usepackage{xcolor}         
\usepackage{MnSymbol}
\usepackage{xspace}

\usepackage{tikz}
\usetikzlibrary{decorations.pathreplacing}

\usepackage{pict2e,picture,graphicx}
\usepackage[overload]{empheq}
\usepackage{microtype}
\usepackage{graphicx}
\usepackage{subcaption}
\usepackage{booktabs}
\usepackage{amsmath}
\usepackage{cases}
\usepackage{mathtools}
\usepackage{amsthm}
\usepackage{thm-restate}
\usepackage{enumitem}

\allowdisplaybreaks
\usepackage{xparse}
\usepackage[capitalize, noabbrev]{cleveref}
\usepackage{wrapfig}
\usepackage{bbm}
\usepackage{yfonts}
\usepackage{nicefrac} 
\usepackage{multirow}
\usepackage{bm}
\usepackage{bbm}
\usepackage{subcaption}

\hypersetup{
	colorlinks = true,
	linkcolor = black,
	citecolor = black,
	linktocpage = true,
	urlcolor = darkGreen
}
\usepackage{comment}

\theoremstyle{plain}
\newtheorem{theorem}{Theorem}[section]

\newtheorem{lemma}[theorem]{Lemma}

\newtheorem{corollary}[theorem]{Corollary}
\newtheorem*{theorem-non}{Theorem}
\newtheorem*{corollary-non}{Corollary}

\theoremstyle{definition}
\newtheorem{definition}[theorem]{Definition}

\theoremstyle{remark}

	\definecolor{niceRed}{HTML}{BE2626}
	\definecolor{Red2}{HTML}{DB3236}
	\definecolor{mgreen}{HTML}{A0C88C}
	\definecolor{blueGrotto}{HTML}{059DC0}
	\definecolor{limeGreen}{HTML}{81B622}
	\definecolor{myellow}{HTML}{E09B23}
	\definecolor{darkGreen}{HTML}{2E8B57}
	\definecolor{navyBlueP}{HTML}{03468F}
	\definecolor{Sepia}{HTML}{7F462C}
	\definecolor{orange2}{HTML}{FF8000}
	\definecolor{mgray}{HTML}{ABB3B8}
	\definecolor{lgray}{HTML}{E5E8E9}
	\definecolor{myPurple}{HTML}{AF007C}
	\definecolor{mypurple2}{HTML}{CC9EFF}
	\definecolor{royalBlue}{HTML}{057DCD}
	\definecolor{mpink}{HTML}{FC6C85}
	\definecolor{lblue}{HTML}{4A90E2}
	\definecolor{peagreen}{HTML}{98C127}
	\definecolor{typnavy}{HTML}{001F3F}
	\definecolor{typblue}{HTML}{0074D9}
	\definecolor{typaqua}{HTML}{7FDBFF}
	\definecolor{typteal}{HTML}{39CCCC}
	\definecolor{typeastern}{HTML}{239DAD}
	\definecolor{typpurple}{HTML}{B10DC9}
	\definecolor{typfuchsia}{HTML}{F012BE}
	\definecolor{typmaroon}{HTML}{85144B}
	\definecolor{typred}{HTML}{FF4136}
	\definecolor{typorange}{HTML}{FF851B}
	\definecolor{typyellow}{HTML}{FFDC00}
	\definecolor{typolive}{HTML}{3D9970}
	\definecolor{typgreen}{HTML}{2ECC40}
	\definecolor{typlime}{HTML}{01FF70}
	\definecolor{newgreen}{HTML}{83C702}
	\definecolor{mathchaYellow}{HTML}{F8E71C}
	\definecolor{mathchaGreen}{HTML}{7ED321}
	\definecolor{mathchaPurple}{HTML}{BD10E0}
	\definecolor{mathchaBlue}{HTML}{58B1FF}
	\definecolor{mathchaRed}{HTML}{FF5B59}
	\definecolor{metablue}{HTML}{0064E0}
\newcommand{\cD}{\mathcal{D}}
\newcommand{\cG}{\mathcal{G}}

\usepackage{enumitem}
\usepackage[style=alphabetic,natbib=true,maxcitenames=2, maxbibnames=10,backend=bibtex]{biblatex}
\allowdisplaybreaks

\usepackage{algorithm}
\usepackage{algpseudocode}
\usepackage{bbm}
\usetikzlibrary{decorations.markings}

\allowdisplaybreaks

\algnewcommand{\Output}[1]{%
  \Statex \hspace*{-\algorithmicindent}\hspace{-0.5mm}\textbf{Output:} #1%
}

\usepackage{cleveref}

\usepackage[dvipsnames]{xcolor}

\newcommand{\OPT}{\textsf{OPT}\xspace}

\usepackage[english]{babel}

\title{Breaking the $T^{3/4}$ Barrier for Regret Minimization With Bi-Dimensional CDFs }

 \author{
 Matteo Castiglioni  \quad 
Anna Lunghi \quad
Alberto Marchesi\vspace{6mm}\\
Politecnico di Milano \\\vspace{6mm}
{\textcolor{black}{\small\texttt{name.surname@polimi.it} }}}

\date{}

\begin{document}

\maketitle

\begin{abstract}

We study regret minimization for learning CDF-related objectives of the form
\[
g(x)\cdot\mathbb{P}_{X\sim\mathcal{D}}(X\le x),
\]
over $[0,1]^2$, where $g$ is a known Lipschitz function and $\mathcal{D}$ is an unknown distribution. At each round $t$, the learner selects a point $x_t$ and observes the binary feedback $\mathbb{I}(X_t\le x_t)$, where $X_t\sim\mathcal{D}$. We design an algorithm achieving regret $\widetilde{\mathcal{O}}(T^{7/10})$, improving over the previous best-known bound of $\widetilde{\mathcal{O}}(T^{3/4})$ and showing that the curse of dimensionality can be at least partially lifted for this class of objectives, though a gap remains with the $\Omega(T^{2/3})$ lower bound. As an application, our techniques yield the same $\widetilde{\mathcal{O}}(T^{7/10})$ regret bound for profit maximization in repeated bilateral trade with fixed prices.

\end{abstract}


\section{Introduction}

We study regret minimization for learning cumulative distribution functions (CDFs) over bi-dimensional domains. In the general $d$-dimensional variant, the learner seeks to learn a function $r:[0,1]^d\to[-1,1]$ of the form
\[ r(x) \coloneq g(x) \cdot  \mathbb{P}_{X\sim \cD} (X\le x),\]
where $g:[0,1]^d\rightarrow [-1,1]$ is a known Lipschitz function, $X$ is drawn from an unknown distribution $\mathcal{D}$, and the inequality $X\le x$ holds component-wise. 
At each round $t\in[T]$, an independent sample $X_t\sim\mathcal{D}$ is
drawn, the learner selects a query point $x_t\in[0,1]^d$, and observes the single bit
$\mathbb{I}(X_t\le x_t)$. The goal is to
minimize the cumulative regret
\[R_T\coloneq \max_{x\in [0,1]^d} \mathbb{E}\left[\sum_{t\in[T]} \left(  r(x)- r(x_t) \right)\right].\] 

This formulation captures numerous economic settings. For instance, it models profit maximization, where the Lipschitz function $g(\cdot)$ represents the profit generated by a trade at a fixed set of prices, while the CDF characterizes the demand curve, i.e., the probability that buyers or sellers accept a trade at that price.

Our problem sits at the intersection of two lines of research. From an online learning perspective, the state of the art for problems of this type is given by (one-sided) Lipschitz bandits, which yield regret $\widetilde{\mathcal{O}}\left(T^{2/3}\right)$ in the univariate case, $\widetilde{\mathcal{O}}\left(T^{3/4}\right)$ in the bi-dimensional case, and $\widetilde{\mathcal{O}}\left(T^{\frac{d+1}{d+2}}\right)$ in general \citep{bubeck2008online,kleinberg2019bandits}. From a more economic perspective, our framework captures several natural learning problems for fixed-price mechanisms, including bilateral trade~\citep{cesa2021regret}, the joint ads model~\citep{aggarwal2024selling}, and small-market models~\citep{lunghi2025online,babaioff2024learning}. The canonical example is profit maximization with fixed-price mechanisms in bilateral trade, for which $\widetilde{\mathcal{O}}\left(T^{3/4}\right)$ is again the best known bound. A substantial line of work studies regret minimization in this setting~\citep{cesa2021regret,cesa2024bilateral,azar2022alpha, bernasconi2024no, lunghi2025better}; most of it addresses the problem of maximizing the gain from trade\footnote{\cite{di2025nearly,di2026profit} study profit maximization, but with general rather than fixed-price mechanisms.} --- a problem not subsumed by our framework --- and likewise achieves a matching $\widetilde{\mathcal{O}}(T^{3/4})$ bound.

All of these problems face a curse of dimensionality: the best known bi-dimensional bounds are worse than the $\widetilde{\mathcal{O}}\left(T^{2/3}\right)$ rate that is optimal in univariate settings like pricing \citep{kleinberg2003value}. In this paper, we focus on the bi-dimensional setting, $d=2$, and ask whether this curse is inherent to CDF-based objectives, or whether it can be at least partially lifted:

\begin{quote}
    Is it possible to improve over the $\widetilde{\mathcal{O}}\left(T^{3/4}\right)$ regret bound for profit maximization with fixed-price mechanisms, and more generally for CDF-based objectives, reducing the curse of dimensionality?
\end{quote}

A related result was recently established for the associated sample-complexity problem by \citet{castiglioni2026sample}, who show that uniformly estimating the CDF of a $d$-variate random variable on $[0,1]^d$ to precision $\epsilon$, over a grid of $K^d$ points, requires only $\frac{1}{\epsilon^3}\log(K)^{\mathcal{O}(d)}$ samples --- a bound that is essentially free of the curse of dimensionality, matching the optimal univariate rate up to polylogarithmic factors. Via a standard explore-then-commit reduction, this result already implies a $\widetilde{\mathcal{O}}\left(T^{3/4}\right)$ regret bound for our class of functions in any fixed dimension. While this improves on the naive bound obtained by running a standard regret minimizer over a uniform grid whenever $d\ge 3$, it falls short of breaking the $\widetilde{\mathcal{O}}\left(T^{3/4}\right)$ barrier, and remains far from the $\widetilde{\mathcal{O}}(T^{2/3})$ rate known to be tight in the univariate case. Dimension $d=2$ is, in this sense, the least satisfying case: it is the only one in which the algorithm of \citet{castiglioni2026sample} fails to improve at all over the trivial bound.

Whether the regret rate in multi-dimensional settings can be pushed below $\widetilde{\mathcal{O}}(T^{3/4})$ was left open. We take a first step toward closing the gap between the upper bound $\widetilde{\mathcal{O}}\left(T^{3/4}\right)$ and the lower bound $\Omega\left(T^{2/3}\right)$ by designing a new algorithm for $d=2$ that achieves regret $\widetilde{\mathcal{O}}\left(T^{7/10}\right)$, disproving the conjecture that explore-then-commit is optimal for every $d\ge 2$.
We view this result as opening more questions than it closes. Is $\widetilde{\mathcal{O}}\left(T^{7/10}\right)$ tight for $d=2$, or is $\widetilde{\mathcal{O}}\left(T^{2/3}\right)$ achievable? Can our approach be extended to arbitrary dimension, or does the $\widetilde{\mathcal{O}}\left(T^{3/4}\right)$ barrier become tight for, say, $d=3$? We leave these as natural directions for future work.

\subsection{Main Contributions}

Our main contribution is a new algorithm for regret minimization with Lipschitz, CDF-dependent rewards in two dimensions that, for the first time, narrows the gap between the $\tilde{\Omega}(T^{2/3})$ lower bound inherited from the univariate case \cite{kleinberg2003value,paseLeme} and the trivial $\tilde{O}(T^{3/4})$ upper bound achievable by running a standard multi-armed bandit algorithm over a uniform grid~\citep{bubeck2008online,kleinberg2019bandits}.
 
\begin{theorem-non}[Informal version of \Cref{thm:main}]
    There exists an algorithm that guarantees regret
    $\tilde{\mathcal{O}}(T^{7/10})$ for reward functions of the form
    $r(x)=g(x) \cdot  \mathbb{P}_{X\sim \cD} (X\le x)$ over $[0,1]^2$.
\end{theorem-non}

Crucially, this guarantee holds in the hardest possible stochastic setting: the learner receives only one bit of feedback per round ($\mathbb{I}(X_t\le x_t)$), and no assumption is made on the correlation or regularity of the underlying distribution $\mathcal{D}$. Since bilateral trade is a special case of our framework, this theorem immediately yields an analogous improvement for fixed-price profit maximization in bilateral trade.

\begin{corollary-non}[Informal version of \Cref{cor:bilateralTrade}]
 There exists an algorithm that guarantees regret
    $\tilde{\mathcal{O}}(T^{7/10})$ for profit maximization in repeated bilateral
    trade with fixed prices and one-bit feedback.
\end{corollary-non}
Analogous results hold for the less-studied settings with two buyers (and no sellers) or two sellers (and no buyers).
To the best of our knowledge, this is the first result to break the $\widetilde{\mathcal{O}}(T^{3/4})$ barrier for fixed-price mechanisms in two dimensions, and the first evidence that the curse of dimensionality afflicting CDF-based regret minimization problems might be avoidable --- much as it was shown to be absent in the sample-complexity setting by \citet{castiglioni2026sample}.

\subsection{Challenges and Techniques}
 
The starting point of our algorithm is the sample-complexity result of \citet{castiglioni2026sample}. However, achieving $o(T^{3/4})$ regret requires departing from the explore-then-commit framework intrinsic to that approach.

\paragraph{A Failed Approach: Uniform Approximation Over an Incomplete Grid.}
 
The natural way to extend a sample complexity result and obtain optimal bounds is the general template of \emph{successive arm elimination}. Time is partitioned into epochs of doubling length, and in epoch $k$ we invoke the uniform CDF-estimation guarantee \citep{castiglioni2026sample} at a target accuracy $\epsilon_k$ that shrinks geometrically with $k$, restricted to the current set of \emph{active} grid points $\mathcal{A}_k$ rather than the full domain, starting from a sufficiently fine uniform grid $\mathcal{A}_0$. At the end of epoch $k$, we eliminate every point whose upper confidence bound falls below the lower confidence bound of the best surviving point. Since every surviving point has true reward within $O(\epsilon_k)$ of optimal, and recalling that the uniform approximation procedure requires $\mathcal{O}(\epsilon_k^{-3})$ rounds (ignoring the size of the grid that appears only in polylogaritmic factors), the regret incurred during epoch $k$ is at most $O(\epsilon_k)$ times the length of the epoch, i.e., ${\mathcal{O}}(\epsilon_k^{-3})$. Since the final target accuracy is $\epsilon\approx T^{-1/3}$, summing over all epochs would yield a regret bound of $\widetilde{\mathcal{O}}(T^{2/3})$!

Unfortunately, we do not have such a uniform exploration procedure on arbitrary subsets: the procedure of \citet{castiglioni2026sample} works only on the full uniform grid, and extending it to an arbitrary subset of the grid is a non-trivial goal that we were not able to achieve. For this reason, we instead pursue a weaker result, relaxing the goal along several directions: we move from a logarithmic number of phases to a two-phase approach, we estimate differences between CDF values at pairs of points rather than the values themselves, and we compute such differences only between a restricted subset of point pairs.

\paragraph{Coarse Pruning via Uniform Approximation.} In the spirit of successive arm elimination by epochs, we run a single preliminary exploration round at coarse precision $\Delta$. Since the initial grid is the full grid, we can directly apply the uniform approximation result \citep{castiglioni2026sample}: this phase runs for $\tilde{\mathcal{O}}(\Delta^{-3})$ rounds and incurs regret at most $\tilde{\mathcal{O}}(\Delta^{-3})$, yielding a $\Delta$-approximation of the CDF uniformly over the whole grid.

\paragraph{Induced Graphs and Smoothing.}
 
As discussed above, we cannot restrict to the set of $\Delta$-optimal points and perform uniform exploration over them efficiently. We can, however, use the coarse estimates to classify grid points by their approximate reward. Points whose estimated reward is far from the estimated optimum are discarded outright, since querying them would be too costly in regret. Points whose estimated reward is close to the estimated optimum are retained as candidates, since they might still be the true optimum. This leaves a third, intermediate class: points that are provably suboptimal, but only mildly so.

The natural instinct is to discard this intermediate class as well, since keeping suboptimal points seems only to inflate the size of the working grid. This intuition is wrong. Ideally, we would like to operate on the whole uniform grid, whereas the candidate set produced by the first pass is typically a sparse and irregularly shaped subset of it. Retaining some points from the intermediate class lets us improve the regularity of this subset.

We formalize regularity through an induced graph: informally, two points are connected if there exists an L-shaped path between them --- a horizontal segment followed by a vertical one, or vice versa --- that lies entirely within the incomplete grid. We decide whether to keep or discard each intermediate point so as to smooth this graph and increase its connectivity; technically, we aim to minimize the independence number of the resulting graph.

\paragraph{Relative Learning on an Incomplete Grid.}
 
Rather than learning the CDF uniformly over the incomplete grid, we lower our ambition and instead learn the \emph{difference} in CDF value between connected points up to a precision $\epsilon$. Our procedure is inspired by \citet{castiglioni2026sample}, but requires new techniques.
First, since queries are restricted to the incomplete grid, the binary search underlying \cite{castiglioni2026sample} --- which bisects intervals at their geometric midpoint --- must be replaced with an index-based variant that only ever bisects at points present in the grid.
Second, we show that the difference between two \emph{arbitrary} grid points --- not just the distance from a point to the origin --- can still be decomposed into $\mathcal{O}(\log K)$ dyadic sub-intervals; this extends the decomposition of \citet{castiglioni2026sample}, which applies only to intervals anchored at the origin, to intervals anchored at two arbitrary endpoints. Whenever the entire interval lies within the incomplete grid, we can recover such an approximation. Combining a horizontal segment with a vertical one then lets us approximate the difference $D(x;y)$ between the CDFs of two connected points $x$ and $y$.

\paragraph{From Relative to Absolute CDF via Feedback Graphs with Misspecified Rewards.}
 
The relative estimates from the previous step almost, but not quite, place us in a standard bandit-with-feedback-graph setting: pulling an arm $x$ should reveal the reward of every neighboring arm $y$ in the induced graph, via the relation $r(y) =g(y)\left(\mathbb{P}_{X\sim \cD}(X\le x)+D(x;y)\right)$. 
The catch is that $D(x;y)$ is itself only known up to an additive error $\epsilon$, so each inferred reward is \emph{biased}, not exact --- a setting not covered by existing feedback-graph analyses.
 
We address this by introducing and analyzing a general model of multi-armed bandits with feedback graphs and $\epsilon$-misspecified rewards, in which every edge of the graph carries a relation between the rewards of its endpoints, known only up to a uniform $\epsilon$-approximation. We design a UCB-style algorithm for this setting and, proceeding along the lines of the standard analysis, show that its regret scales as $\tilde{\mathcal{O}}\left(\sqrt{\alpha(\mathcal{R})\,T} + \epsilon\,T\right)$, where $\alpha(\mathcal{R})$ is the independence number of the feedback graph. This recovers the standard bound of the classical (unbiased) feedback-graph setting, while the bias of the relative estimates contributes only an additional, separable $\mathcal{O}(\epsilon\,T)$ term to the regret. This result may be of independent interest beyond our application.

Balancing the coarse-pruning precision $\Delta$ against the target accuracy $\epsilon$ of the relative-learning phase completes the proof, yielding a total regret of $\widetilde{\mathcal{O}}\left(T^{7/10}\right)$.

\subsection{Related Works}
Here, we review the lines of research most closely connected to our work.
 
\paragraph{GFT Maximization in Bilateral Trade.} Bilateral trade is a classical mechanism design problem modeling the exchange of a good between a buyer and a seller through an intermediary. Motivated by the impossibility of designing a mechanism that is simultaneously incentive compatible, individually rational, fully efficient, and budget balanced~\cite{MyersonS83}, a recent line of research studies how to learn optimal mechanisms in repeated bilateral trade. \citet{cesa2021regret} initiated the study of regret minimization in this setting, focusing on the gain-from-trade (GFT) objective. A rich literature extends these results along many directions, providing a complete characterization of the setting~\cite{azar2022alpha,cesa2023repeated, bernasconi2024no,lunghi2025better,chen2025tight,gaucher2025featurebased,cosson2026contextual,coccianonparametric,bachoc2024fair}. Interestingly, most of these works establish $\widetilde{\mathcal{O}}\left(T^{3/4}\right)$ regret bounds, showing that gain-from-trade maximization also suffers from a curse of dimensionality. It is important to remark, however, that this arises for reasons quite different from ours, stemming from the absence of bandit feedback and from connections to the apple-tasting problem.
 
\paragraph{Profit Maximization in Bilateral Trade.} Until recently, the profit-maximization problem had received comparatively little attention. \citet{di2025nearly,di2026profit} study the problem of maximizing profit in repeated bilateral trade, but with general mechanisms. Their setting and results are conceptually quite different from ours: the main challenge in learning a fixed price is limited feedback, whereas with general mechanisms the learner has full feedback but a huge action set. Their main results establish a $\widetilde{\mathcal{O}}\left(\sqrt{T}\right)$ regret bound in the stochastic setting, and an analogous bound against a smoothed adversary, respectively.
 
\paragraph{CDF Learning and Pricing.} A parallel line of work studies the problem of estimating a CDF from limited feedback, motivated by applications to dynamic pricing. In the simpler dynamic pricing problem, the learner repeatedly posts a price and observes only whether a buyer accepts it, which is equivalent to querying the CDF of the buyer's value distribution at a single point per round. Most work on this topic is one-dimensional; the most relevant result for us is the characterization of the optimal sample-complexity bound $\Theta\left(\epsilon^{-3}\right)$ and the corresponding optimal regret bound $\Theta(T^{2/3})$~\cite{kleinberg2003value,paseLeme}. More recently, \citet{castiglioni2026sample} study the sample complexity of learning a CDF uniformly over an arbitrary fine grid, providing a bound of $\epsilon^{-3}\log(K)^{\mathcal{O}(d)}$ samples, where $d$ is the number of dimensions. This yields applications to pricing with multiple buyers and/or sellers, including bilateral trade. Their result can also be extended to a regret bound via a standard explore-then-commit approach, yielding $\widetilde{\mathcal{O}}(T^{3/4})$ regret --- the first non-trivial bound in this direction.

\section{Preliminaries}

We study regret minimization over the decision space \( [0,1]^2 \), where rewards are defined through the \emph{cumulative distribution function} (CDF) of an unknown probability distribution. Let \( \mathcal{D} \) be a probability distribution supported on \( [0,1]^2 \). 
Its CDF is the function $F:[0,1]^2\to[0,1]$ defined by
\[F(x) \coloneq \mathbb{P}_{X\sim \cD}(X\le x),\]
where the inequality holds component-wise.
Let \( g : [0,1]^2 \to [-1,1] \) be a \emph{known} \(1\)-Lipschitz function with respect to the \( \ell_1 \)-norm,\footnote{Notice that the assumption that the Lipschitz constant is equal to \(1\) is without loss of generality. All our results extend straightforwardly to arbitrary Lipschitz constants. Moreover, since all norms are equivalent up to constants depending only on $d$, and we work in fixed dimension $d=2$, our results extend straightforwardly to arbitrary norms, at the cost of only a multiplicative constant in the regret bound.} \emph{i.e.}, satisfying \( |g(x)-g(y)| \le \|x-y\|_1 \) for every \( x, y \in [0,1]^2 \).

The learner interacts with the environment over \( T \in \mathbb{N} \) rounds. At each round \( t \in [T] \),\footnote{We denote by \( [n] \coloneq \{1,\ldots,n\} \) the set of the first \( n \in \mathbb{N} \) natural numbers.} the learner selects a point \( x_t \in [0,1]^2 \), the environment independently draws a sample \( X_t \sim \mathcal{D} \), and the learner receives a reward \( r_t(x_t) \) determined by a reward function \( r_t : [0,1]^2 \to [-1,1] \) defined by:
\[r_t(x) \coloneq g(x) \cdot  \mathbb{I}(X_t\le x).\]
Finally, the learner observes the single bit:
\[\mathbb{I}(X_t\le x_t).\]
Note that, while \( g \) is known to the learner, the distribution \( \mathcal{D} \) is not. Crucially, the single observed bit \( \mathbb{I}(X_t \le x_t) \) contains enough information to learn \( F \), whereas it would be uninformative about \( g \).
The learner's goal is to minimize its \emph{(cumulative) regret} over the $T$ rounds, defined as:
\begin{equation*}
    R_T \coloneq \max_{x\in [0,1]^2} \mathbb{E}\left[\sum_{t=1}^T r_t(x)-\sum_{t=1}^T r_t(x_t)\right].
\end{equation*}

\subsection{Fixed-Price Mechanisms and Profit Maximization in Bilateral Trade}\label{sec:bilateral}

A canonical class of problems captured by our framework is profit maximization in repeated posted-price problems. In the classical model of \citet{kleinberg2003value}, the learner (seller) posts a price $p_t \in [0,1]$ to a single buyer, who accepts to buy if their valuation $V_t \in [0,1]$ exceeds $p_t$, yielding profit $p_t$ to the learner.
This is an instance of our framework with $x_t=1-p_t$, $g(x)=1-x$ and $X_t=1-V_t$, so that $g(x)=p_t$ and $\mathbb{I}(X_t \le x_t) = \mathbb{I}(V_t \ge p_t)$. It is well known that, in this one-dimensional problem, a minimax regret bound of \( \widetilde{\Theta}(T^{2/3}) \) can be achieved \citet{kleinberg2003value}.

A natural two-dimensional generalization of this problem arises in profit-maximization settings involving two agents (buyers and/or sellers), among which bilateral trade is perhaps the most prominent and influential model~\citep{cesa2021regret}.
At each round \( t \in [T] \), a new buyer and a new seller arrive, each with a private valuation for the indivisible item to be exchanged, denoted by \( B_t \in [0,1]\) and \( S_t\in [0,1] \), respectively. The learner then proposes two prices, \( p_t \in [0,1] \) to the seller and \( q_t \in [0,1] \) to the buyer. Both the seller and the buyer may decide whether to accept the proposed trade.
A trade occurs if and only if \( S_t \le p_t \) and \( B_t \ge q_t \), in which case the learner obtains profit \( q_t - p_t \). The goal of the learner is to maximize cumulative profit.

By setting 
\[X_t=(S_t,1-B_t), \quad x_t=(p_t,1-q_t), \quad \textnormal{and}\quad g(x_1,x_2)=1- x_1-x_2,
\]
we embed this problem in our framework, since  $g(x_t)=q_t-p_t$ and
$\mathbb{I}(X_t\le x_t)=\mathbb{I}(S_t\le p_t,\,B_t\ge q_t)$.

\subsection{Grids and Their Induced Graphs}

In this paper, we work extensively with uniform grids defined over $[0,1]^2$ and graphs induced by subsets thereof.
Given any $K \in \mathbb{N}$, we define the one-dimensional uniform grid
\[
G^U(K) \coloneq \left\{\frac{i}{K-1} \mid i \in \{0,\ldots, K-1\}\right\},
\] and the two-dimensional uniform grid
\[
\mathcal{G}^U(K) \coloneq G^U(K)\times G^U(K).
\]

We also work with segments connecting points on the grid $\cG^U(K)$. Given $a,b \in [0,1]$, for ease of notation we define the interval 
\[I(a,b)\coloneq ( \min\{a,b\},\max\{a,b\} ] \]
so that the notation is symmetric in $a$ and $b$, namely $I(a,b) = I(b,a)$.
Given two points $x = (x_1,x_2), y = (y_1,y_2) \in \mathcal{G}^U(K)$ that share one coordinate (\emph{i.e.}, either $x_1 = y_1$ or $x_2= y_2$), we define the segment connecting them as the set $\mathfrak{I}(x,y) \subseteq \mathcal{G}^U(K)$ of grid points lying on the
axis-parallel line connecting them:
\[
\mathfrak{I}(x,y) \coloneq \begin{cases} \left( x_1\times  I(x_2,y_2) \right) \cap  \cG^U(K) \quad \text{if } x_1=y_1 \\
 \left( I(x_1,y_1) \times x_2 \right) \cap  \cG^U(K) \quad \text{if }x_2=y_2, \end{cases}
\]
where we ignore the dependency on $K$ since it will be clear from context.
Notice that the segment $\mathfrak{I}(x,y)$ is well defined only when $x$ and $y$ share one coordinate.

Given two points $x=(x_1,x_2),y=(y_1,y_2) \in \cG^U(K)$, we define the set of ``corner-points'' as
\begin{align} \label{eq:corner}
T(x,y)\coloneq\{(x_1,y_2),(y_1,x_2)\},
\end{align}
each of which connects the two points through a triangular structure. 
Each element of $T(x,y)$ is axis-aligned with both $x$ and $y$, so that the two segments forming an L-shaped path from $x$ to $y$ through a corner $z\in T(x,y)$ are both well defined.

Then, we are finally ready to define the graph induced by a grid $\cG$: two points $x,y \in \cG$ are connected if there exists an L-shaped path through a corner $z\in T(x,y)$ which stays entirely within $\mathcal{G}$. Formally:
\begin{definition} \label{def:graph}
 Given a subset of the uniform grid $\cG \subseteq \cG^U(K)$, the induced graph $\mathcal{R}(\cG)$ has vertex set $\cG$ and contains an edge between $x \in \cG$ and $y \in \cG$ if and only if there is a corner $z \in T(x,y)$ such that:
 \[ \mathfrak{I}(x,z)\cup \mathfrak{I}(y,z)\subseteq \cG . \]
\end{definition}

\Cref{fig:induced-graph} illustrates several examples of edges and non-edges.
A key quantity throughout the paper is the \emph{independence number}
$\alpha(\mathcal{R}(\mathcal{G}))$, defined as the size of a largest independent set in
$\mathcal{R}(\mathcal{G})$.

\begin{figure}
\centering
\begin{tikzpicture}[scale=4]

\def\step{0.25}



\foreach \i in {0,1,2,3,4}{
    \draw[gray!20, thin] ({\i*\step},0) -- ({\i*\step},1);
    \draw[gray!20, thin] (0,{\i*\step}) -- (1,{\i*\step});
}

\def\missing{(0.25,0.5),(0.5,0.5),(0.75,0.5)}

\foreach \i in {0,1,2,3,4}{
    \foreach \j in {0,1,4}{
        \fill[black] ({\i*\step},{\j*\step}) circle (1.5pt);
    }
}

\fill[black] (0,0.5) circle (1.5pt);
\fill[black] (0,0.75) circle (1.5pt);
\fill[black] (0.5,0.75) circle (1.5pt);
\fill[black] (0.75,0.75) circle (1.5pt);
\fill[black] (1,0.75) circle (1.5pt);

\fill[gray!70,thin] (0.25,0.5) circle ((0.8pt);
\fill[gray!70,thin] (0.5,0.5) circle ((0.8pt);
\fill[gray!70,thin] (0.75,0.5) circle ((0.8pt);
\fill[gray!70,thin] (1,0.5) circle ((0.8pt);

\fill[gray!70,thin] (0.25,0.75) circle ((0.8pt);

\draw[green!60!black, thick] (0,0.25) -- (0.5,0);
 \fill[green!50, opacity=0.3] (0,0.25) -- (0.5,0) -- (0,0) -- cycle;
;

\draw[green!60!black, thick] (0,0.75) -- (0.25,1);
\fill[green!50, opacity=0.3] (0,0.75) -- (0,1)-- (0.25,1) -- cycle;

\draw[red!70, thick, dashed] (0,0.25) -- (0.5,0.75);

\draw[green!60!black, thick] (0.75,0) -- (0.75,0.25);

\draw[red!70, thick, dashed] (0,0.25) -- (1,0.75);

\foreach \i/\lab in {0/0, 1/$\frac{1}{K{-}1}$, 2/$\frac{2}{K{-}1}$, 3/$\frac{3}{K{-}1}$, 4/$1$}{
    \node[below, font=\tiny] at ({\i*\step}, -0.04) {\small\lab};
    \node[left,  font=\tiny] at (-0.04, {\i*\step}) {
    \small\lab};
}

\node[right, font=\scriptsize] at (1.05, 0.9) {\textbullet\ $\in \mathcal{G}$};
\node[right, font=\scriptsize, gray!60] at (1.05, 0.75) {\textbullet\ $\notin \mathcal{G}$};
\node[right, font=\scriptsize, green!60!black] at (1.05, 0.6) {--- edge in $\mathcal{R}(\mathcal{G})$};
\node[right, font=\scriptsize, red!70] at (1.05, 0.45) {- - no edge};

\draw[black, thin] (0,0) rectangle (1,1);

\end{tikzpicture}
\caption{Examples of existing and non-existing edges in the graph induced by a grid $\cG$.}\label{fig:induced-graph}
\end{figure}
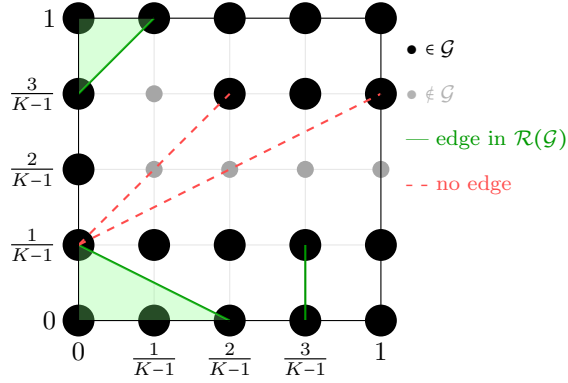

\subsection{Basics on Probability Estimation and Uniform Learning}

We rely on two procedures that have been recently introduced in \citet{castiglioni2026sample}.
The first procedure, which we call $\textsf{Estimate}$ and which is called $\textsc{MCE}$ in \Citep{castiglioni2026sample}, estimates the probability mass of an axis-aligned rectangle.

\begin{restatable}{lemma}{lemmaCastiglioni}[Lemma 4.2 in \cite{castiglioni2026sample}]
\label{lem:Estimate}
    There exists an algorithm that, given a confidence $\delta\in(0,1)$ and an axis-aligned rectangle $(a,b]\times[0,w]\subseteq[0,1]^2$, uses at most $1/\epsilon^2$ rounds\footnote{Throughout the paper, we omit the ceiling operator in round numbers and assume that they are integer quantities.} and outputs $\widehat{P}((a,b]\times[0,w])$ such that, with probability at least $1-\delta$:
    \[
        \bigl|\widehat{P}((a,b]\times[0,w])
              -\mathbb{P}_{X\sim\mathcal{D}}(X\in(a,b]\times[0,w])\bigr|
        \;\le\;
        2\epsilon\sqrt{\tfrac{1}{2}\ln\tfrac{2}{\delta}}.
    \]
\end{restatable}
The second procedure is the main result of \citep{castiglioni2026sample}: a uniform learning procedure, which we call \textsf{UniformCDF}, that simultaneously estimates the CDF at every point of a grid.
\begin{theorem}[Bi-dimensional Case of Theorem 3.1 in \Citep{castiglioni2026sample}]\label{theo: mainOLD}
    There exists an algorithm that, given an accuracy $\Delta > 0$, a confidence $\delta \in (0,1)$, and a uniform grid $\mathcal{G}^U(K)$ over $[0,1]^2$ of resolution $1 / K$ (with $K \in \mathbb{N}$), uses at most $\frac{1}{\Delta^3}\log(K/\delta)^{\mathcal{O}(1)}$ rounds and outputs a function $\widetilde P: \mathcal{G}^U(K) \to [0,1]$ such that, with probability at least $1-\delta$, the following holds:
      \[\left|\mathbb{P}_{X \sim \cD}(X \le x )-\widetilde{P}(x)\right|\le \Delta \quad \forall x \in \mathcal{G}^U(K).\]
\end{theorem}


\section{Main Results and Proof Plan}

The main contribution of this paper is the following theorem.

\begin{restatable}{theorem}{thmmain} \label{thm:main}
    Let $\{X_t\}$ be i.i.d. samples in $[0,1]^2$. There exists an algorithm that guarantees 
    \[R_T= \widetilde{\mathcal{O}}\left( T^{\frac{7}{10}}\right).\]
\end{restatable}

Since bilateral trade is a special case of our framework
(see \Cref{sec:bilateral}), \Cref{thm:main} immediately yields the following.

\begin{corollary}\label{cor:bilateralTrade}
 There exists an algorithm for profit maximization in repeated bilateral trade that guarantees 
    \[R_T= \widetilde{\mathcal{O}}\left( T^{\frac{7}{10}}\right).\]
\end{corollary}

The rest of the paper is devoted to proving \Cref{thm:main}.
The algorithm proceeds in three sequential phases, summarized in
\Cref{alg: main non-bounded}, and its design is guided by the following core tension.

To minimize the regret we face two competing objectives: we want the working grid $\mathcal{G}$ to be as \emph{dense} as possible, so that the uniform exploration in the spirit
of \citep{castiglioni2026sample} can be applied efficiently and the optimal
$\tilde{O}(1/\epsilon^3)$ sample complexity is achieved; but we also want
$\mathcal{G}$ to contain only points that are \emph{not too suboptimal}, since every query to a too suboptimal point incurs a large per-round regret.
The three phases resolve this tension by finding the right compromise between this contrasting goals.

\paragraph{Phase 1 ---  Coarse Pruning via Uniform Approximation.}
We run $\textsc{UniformCDF}(\mathcal{G}^U(K), \Delta, \delta)$
(\Cref{theo: mainOLD}) at a low precision $\Delta$, obtaining a
uniform approximation $\widehat{F}$ of the CDF over $\mathcal{G}^U(K)$ at a cost
of $\widetilde{\mathcal{O}}(1/\Delta^3)$ queries (and hence at most $\widetilde{\mathcal{O}}(1/\Delta^3)$
regret). Note that this approximation alone is clearly not sufficient to break the $\tilde{O}(T^{3/4})$ barrier: running this phase alone with $\Delta = T^{-1/4}$ already yields
$\widetilde{\mathcal{O}}(T^{3/4})$ regret via a standard explore-then-commit argument. Our
improvement instead comes from choosing a much larger $\Delta$ and using this phase only to obtain a rough approximation.

\paragraph{Processing Step --- Building a Highly-Connected Graph Through Smoothing.}
The previous step leaves us with three classes of points: \emph{candidate optima} $\mathcal{G}^{\OPT}$, which must be kept for the following phase; \emph{clearly suboptimal} points, which must be removed since querying them would incur large per-round regret; and a remaining set of \emph{borderline} points, which are not suboptimal enough to require removal, but whose inclusion is optional. Since our goal is simply to minimize the independence number of $\mathcal{G}^\OPT$,the optimal approach  
is straightforward: we improve the connectivity by considering a graph $\mathcal{G}$ that keeps all borderline points that help connect points in $\mathcal{G}^{\OPT}$, discard the rest, and then consider only the subgraph $\mathcal{R}^\OPT$ induced by the vertexes $\mathcal{G}^\OPT$.
Overall, this procedure can be seen as a smoothing of the graph and grid that removes ``sharp,'' isolated components (see Figure \ref{fig:smoothing}).

\paragraph{Phase 2 --- Relative Learning on an Incomplete Grid.}
So far, we have only used the notion of connectivity to build a well-connected graph, on the promise that learning would be easier for more connected graphs. We now put this connectivity to use. Rather than estimating the CDF at each point of $\mathcal{G}$ individually, we instead estimate only the \emph{relative differences}
\[ D(x;y)= F(y)-F(x)\]
between pairs of points connected in $\mathcal{R}(\mathcal{G})$.
The procedure $\textsc{RLS}$ (\Cref{alg: RLS}) outputs estimates
$\widehat{D}(x;y)$ for all edges $(x,y)\in\mathcal{R}(\mathcal{G})$ satisfying,
with high probability,
\[
    \bigl|\widehat{D}(x;y) - D(x;y)\bigr| \le \epsilon
    \qquad\forall\,(x,y)\in\mathcal{R}(\mathcal{G}),
\]
using only $\tilde{O}(1/\epsilon^3)$ queries in total --- independently of
$|\mathcal{G}|$. To do so, it first estimates the distance between axis-aligned points, and then exploits the $L$-shaped structure of edges in
$\mathcal{R}(\mathcal{G})$ to extend these estimates to all edges.

\paragraph{Phase 3 --- Multi-Armed Bandits with Feedback Graph and $\epsilon$-Misspecified Rewards.}
The estimates $\widehat{D}$ from Phase 2 turn $\mathcal{R}^\OPT$ into a
\emph{feedback graph} with $\epsilon$-misspecified rewards: querying a point $x$
yields a slightly biased estimator of the reward of every neighbor $y$ of $x$ in
$\mathcal{R}^\OPT$.
We design a UCB-style algorithm (\Cref{alg: graph}) for this misspecified setting
and show that its regret scales as
\[
    \tilde{\mathcal{O}}\!\left(\sqrt{\alpha(\mathcal{R}^\OPT)\,T} + \epsilon\, T\right),
\]
replacing the usual dependence on the number of arms with the independence number of the feedback graph, at the cost of an additional $\epsilon$ per-round regret due to the bias.

\begin{algorithm}[h]\caption{}\label{alg: main non-bounded}
    \begin{algorithmic}[1]
        \State Input: errors $\epsilon<\Delta$, probability $\delta$ 
        \State $\hat{F}(\cdot)\gets \textsf{UniformCDF}( \mathcal{G}^U(1/\epsilon), \Delta,\delta)$ \Comment{Phase 1}
        \State $(\cG, \mathcal{R}^\OPT)=\textsf{BuildGraph}(\mathcal{G}^U(1/\epsilon),\hat{F}(\cdot),\Delta)$
        \State $\{\hat{D}(\cdot\;;\cdot)\} \gets  \textsf{RLS}(\cG,\epsilon)$ \Comment{Phase 2}
        \State Run \textsf{$\epsilon$-GraphUCB}$(\mathcal{R}^\OPT, \epsilon, \delta)$ \Comment{Phase 3}
    \end{algorithmic}
\end{algorithm}

\section{Build a Highly-Connected Graph Through Smoothing}

\begin{figure}
    \centering
    \begin{subfigure}{0.45\textwidth}
    \centering
\begin{tikzpicture}[scale=3]

\def\step{0.25}



\foreach \i in {0,1,2,3,4}{
    \draw[gray!20, thin] ({\i*\step},0) -- ({\i*\step},1);
    \draw[gray!20, thin] (0,{\i*\step}) -- (1,{\i*\step});
}

\def\missing{(0.25,0.5),(0.5,0.5),(0.75,0.5)}

\foreach \i in {0,1,2,3,4}{
        \fill[black] ({\i*\step},{(4-\i)*\step}) circle (1.5pt);
}
\foreach \i in {0,1,2,3}{
        \fill[gray!70,thin] ({\i*\step},{(4-\i-1)*\step}) circle (0.8pt);
}
\foreach \i in {0,1,2}{
        \fill[gray!70,thin] ({\i*\step},{(4-\i-2)*\step}) circle (0.8pt);
}

\foreach \i in {0,1}{
        \fill[gray!70,thin] ({\i*\step},{(4-\i-3)*\step}) circle (0.8pt);
}

\fill[gray!70,thin] (0,0) circle ((0.8pt);

\foreach \i in {2,3,4}{
        \fill[gray!70,thin] ({\i*\step},{(4-\i+2)*\step}) circle ((0.8pt);
}
\foreach \i in {1,2,3,4}{
        \fill[green!60] ({\i*\step},{(4-\i+1)*\step}) circle ((1.5pt);
}

\foreach \i in {3,4}{
        \fill[gray!70,thin] ({\i*\step},{(4-\i+3)*\step}) circle ((0.8pt);
}
        \fill[gray!70,thin] (1,1) circle ((0.8pt);

\fill[green!60] (0.5,1) circle ((1.5pt);
\fill[green!60] (1,0.5) circle ((1.5pt);

\draw[green!60!black, thick] (0,1) -- (0.5,1);
\draw[green!60!black, thick] (0.5,0.5) -- (0.5,1);
\draw[green!60!black, thick] (1,0.5) -- (0.5,0.5);
\draw[green!60!black, thick] (1,0.5) -- (1,0);

\fill[green!60] (1,1) circle ((1.5pt);

\node[right, font=\scriptsize] at (1.05, 0.9) {\textbullet\ $\in \mathcal{G}^\OPT$};
\node[right, font=\scriptsize, green!60!black] at (1.05, 0.75) {\textbullet Borderline points};
\node[right, font=\scriptsize,gray!60] at (1.05, 0.6) {\textbullet Clearly suboptimal points};
\end{tikzpicture}
\end{subfigure}
\centering
\begin{subfigure}{0.5\textwidth}
\centering
\begin{tikzpicture}[scale=3]

\def\step{0.25}



\foreach \i in {0,1,2,3,4}{
    \draw[gray!20, thin] ({\i*\step},0) -- ({\i*\step},1);
    \draw[gray!20, thin] (0,{\i*\step}) -- (1,{\i*\step});
}

\def\missing{(0.25,0.5),(0.5,0.5),(0.75,0.5)}

\foreach \i in {0,1,2,3,4}{
        \fill[black] ({\i*\step},{(4-\i)*\step}) circle (1.5pt);
}
\foreach \i in {0,1,2,3}{
        \fill[gray!70,thin] ({\i*\step},{(4-\i-1)*\step}) circle (0.8pt);
}
\foreach \i in {0,1,2}{
        \fill[gray!70,thin] ({\i*\step},{(4-\i-2)*\step}) circle (0.8pt);
}

\foreach \i in {0,1}{
        \fill[gray!70,thin] ({\i*\step},{(4-\i-3)*\step}) circle (0.8pt);
}

\fill[gray!70,thin] (0,0) circle ((0.8pt);

\foreach \i in {2,3,4}{
        \fill[gray!70,thin] ({\i*\step},{(4-\i+2)*\step}) circle ((0.8pt);
}
\foreach \i in {1,2,3,4}{
        \fill[black] ({\i*\step},{(4-\i+1)*\step}) circle ((1.5pt);
}

\foreach \i in {3,4}{
        \fill[gray!70,thin] ({\i*\step},{(4-\i+3)*\step}) circle ((0.8pt);
}
        \fill[gray!70,thin] (1,1) circle ((0.8pt);

\fill[black] (0.5,1) circle ((1.5pt);
\fill[black] (1,0.5) circle ((1.5pt);

\node[right, font=\scriptsize] at (1.05, 0.9) {\textbullet\ $\in \mathcal{G}$};

\node[right, font=\scriptsize,gray!60] at (1.05, 0.75) {\textbullet $\notin \mathcal{G}$};
\end{tikzpicture}
\end{subfigure}
\caption{Graph before and after adding borderline points that improve connectivity to the candidate optima $\cG^{\OPT}$.} \label{fig:smoothing}
\end{figure}
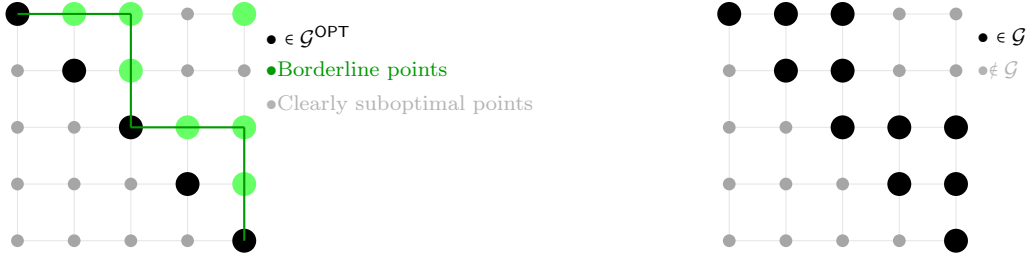

At a high level, the goal of this section is to build an incomplete grid $\mathcal{G}\subseteq \mathcal{G}^U(1/\epsilon)$ that is highly connected and does not contain highly suboptimal points.
 
Our algorithm is presented in \Cref{alg: build graph} and works as follows. For each point $x\in \cG^U(1/\epsilon)$, it computes the estimated reward
\[\tilde r(x)\coloneqq g(x) \hat F(x),\]
and defines $\widetilde{\OPT} \coloneq \max_{x\in \mathcal{G}^U(1/\epsilon)}\tilde r(x)$ as the estimated optimal reward.

Since we do not want to remove the optimal solution from our final grid $\cG$, we need to keep all points that we cannot rule out as optimal, defined as
\[\cG^{\OPT}\coloneqq \{ x\in \cG^U(1/\epsilon): \tilde r(x)\ge \widetilde{\OPT} - 2\Delta\}.\]
We then observe that it is also safe to add points to the final graph $\cG$, provided they are at most $O(\Delta)$-suboptimal and improve connectivity. Given a set $\cG'$, we let
\[\tilde r(\cG')\coloneqq \min_{x \in \cG'} \tilde r(x).\]

Starting from $\mathcal{G}=\cG^{\OPT}$, we iterate over all pairs $x,y \in \cG^\OPT$ and attempt to connect them: this is possible whenever there exists an $L$-shaped path connecting $x$ and $y$ that is at most $6\Delta$-suboptimal with respect to the estimated $\tilde{r}$, in which case we add both segments of the path to $\cG$. Figure \ref{fig:smoothing} shows an example of this procedure.

\begin{algorithm}[h]\caption{\textsf{BuildGraph}}\label{alg: build graph}
    \begin{algorithmic}[1]
        \State Input: grid $\cG^U(1/\epsilon)$, CDF-estimation $\hat F$, error $\Delta$
        \State $\widetilde{\OPT} \gets \max_{x\in \mathcal{G}^U(1/\epsilon)}\;g(x)\hat{F}(x)$
        \State $\mathcal{G}^\OPT\gets \{x\in \mathcal{G}^U(1/\epsilon): g(x)\hat{F}(x)\ge \widetilde{\OPT}-2\Delta \}$
        \If{$\mathcal{G}^\OPT\cap \{x\in \mathcal{G}^U(1/\epsilon): g(x)\ge 0\}\neq \varnothing$}
        \State $\mathcal{G}^\OPT \gets \mathcal{G}^\OPT\cap \{x\in \mathcal{G}^U(1/\epsilon): g(x)\ge 0\}$
        \EndIf
        \State $\cG=\cG^\OPT$
        \For {$x,y\in \cG^\OPT$}
            \If{$\exists z \in T(x,y): \tilde r(\mathfrak{I}(x,z)\cup \mathfrak{I}(y,z))  \ge \widetilde{\OPT}-6\Delta  $}
                \State $\cG= \cG \cup \mathfrak{I}(x,z)\cup \mathfrak{I}(y,z)$
            \EndIf
        \EndFor
        \State Create the graph $\mathcal{R}(\mathcal{G})$ induced by $\mathcal{G}$
        \State Create $\mathcal{R}^\OPT$ subgraph of $\mathcal{R}(\mathcal{G})$ induced by the vertexes $\mathcal{G}^\OPT$
        \State\Return ($\cG,\mathcal{R}^\OPT$)
    \end{algorithmic}
\end{algorithm}

We proceed to analyze the guarantees of the algorithm, which rely on the following clean event.

\begin{definition}[Clean Event $E_0$]
We say the clean event $E_0$ holds if
\[\max_{x\in\mathcal{G}^U(K)}|\widehat{F}(x)-F(x)|\le\Delta.\]
\end{definition}

We now show that, under $E_0$, $\cG^\OPT$ (and hence $\cG$) contains the optimum, and that $\cG$ contains only $O(\Delta)$-suboptimal points.

\begin{lemma}\label{lemma: subopt gap}
Let $\OPT \coloneqq \max_{x\in\mathcal{G}^U(1/\epsilon)} g(x)F(x)$ denote the true optimal value. Under the clean event $E_0$, it holds that
\[ \max_{x \in \cG^\OPT}  g(x) F(x) = \OPT \]
and
\[   g(x) F(x) \ge \OPT - 8 \Delta \quad \forall x \in \cG.\]
\end{lemma}

\begin{proof}
Let $\widetilde{x}^*$ and $x^*$ be such that
\begin{align*}
    \widetilde{x}^* \in  \arg \max_{x\in \mathcal{G}^U(1/\epsilon)}\;g(x)\hat{F}(x), \\
    x^* \in \arg \max_{x\in \mathcal{G}^U(1/\epsilon)}\;g(x){F}(x).
\end{align*}
For every $x\in \mathcal{G}$ we have $g(x)\hat{F}(x)\ge \widetilde{\OPT}-6\Delta$ by construction, which implies
\begin{align*}
    g(x)F(x)&\ge g(x)(F(x)-\hat{F}(x))+ g(x)\hat{F}(x)\ge -|F(x)-\hat{F}(x)| + \widetilde{\OPT}-6\Delta\\
    & \ge -\Delta + g(\widetilde{x}^*)\hat{F}(\widetilde{x}^*)-6\Delta \\
    &\ge g(x^*)\hat{F}(x^*)-7\Delta\\
    &= g(x^*)(\hat{F}(x^*)-F(x^*))+ g(x^*)F(x^*)-7\Delta\\
    & \ge \OPT - 8\Delta.
\end{align*}
This proves the second part of the statement.

Additionally, for every $x\notin \mathcal{G}^{\OPT}$ we have that either $g(x)\hat{F}(x)< \widetilde{\OPT}-2\Delta$ or $g(x)<0$ and there exists $y\in \mathcal{G}^\OPT$ such that $g(y)\ge 0$. If the latter condition is verified, it is immediate to observe that $\OPT \ge g(y)F(y)\ge 0 \ge g(x)F(x)$, and if the first condition is verified  we have 

\begin{align*}
      g(x)F(x)&= g(x)(F(x)-\hat{F}(x))+ g(x)\hat{F}(x)< |F(x)-\hat{F}(x)| + \widetilde{\OPT}-2\Delta\\
    & \le \Delta + g(\widetilde{x}^*)\hat{F}(\widetilde{x}^*)-2\Delta\\
    &= g(\widetilde{x}^*)(\hat{F}(\widetilde{x}^*)-F(\widetilde{x}^*))+ g(\widetilde{x}^*)F(\widetilde{x}^*)-\Delta\\
    & \le  g(\widetilde{x}^*)F(\widetilde{x}^*)\\
    &\le \OPT,
\end{align*}
which shows that $x^*\in \mathcal{G}^\OPT$ for at least one possible choice of $x^*$ and proves the first part of the statement.
\end{proof}

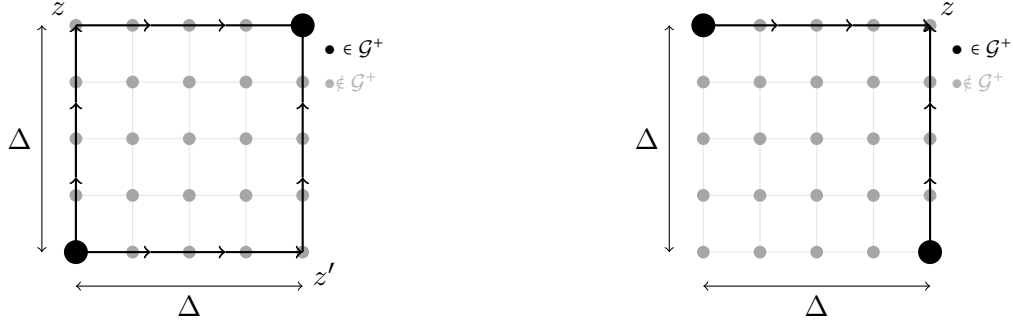
\begin{figure}

\centering
\begin{subfigure}[b]{0.5\textwidth}
\centering
\begin{tikzpicture}[scale=3]

\def\step{0.25}

\foreach \i in {0,1,2,3,4}{
    \draw[gray!20, thin] ({\i*\step},0) -- ({\i*\step},1);
    \draw[gray!20, thin] (0,{\i*\step}) -- (1,{\i*\step});
}
\foreach \i in {0,1,2,3,4}{
\foreach \j in {0,1,2,3,4}{
        \fill[gray!70,thin] ({\i*\step},{(\j)*\step}) circle (0.8pt);
}
}
\fill[black] (0,0) circle (1.5pt);
\fill[black] (1,1) circle (1.5pt);

\node[above left] at (0,1) {$z$};
\node[below right] at (1,0) {$z'$};

\draw[<->, shift={(0,-0.15)}] (0,0) -- (1,0) node[midway, below] {$\Delta$};
\draw[<->, shift={(-0.15,0)}] (0,0) -- (0,1) node[midway, left] {$\Delta$};

\draw[->, thick] (0,0) -- (0.33,0);
\draw[->, thick] (0.33,0) -- (0.66,0);
\draw[->, thick] (0.66,0) -- (1,0);

\draw[->, thick] (1,0) -- (1,0.33);
\draw[->, thick] (1,0.33) -- (1,0.66);
\draw[->, thick] (1,0.66) -- (1,1);

\draw[->, thick] (0,0) -- (0,0.33);
\draw[->, thick] (0,0.33) -- (0,0.66);
\draw[->, thick] (0,0.66) -- (0,1);

\draw[->, thick] (0,1) -- (0.33,1);
\draw[->, thick] (0.33,1) -- (0.66,1);
\draw[->, thick] (0.66,1) -- (1,1);

\node[right, font=\scriptsize] at (1.05, 0.9) {\textbullet\ $\in \mathcal{G}^+$};
\node[right, font=\scriptsize,gray!60] at (1.05, 0.75) {\textbullet $\notin \mathcal{G}^+$};
\end{tikzpicture}
\end{subfigure}
\begin{subfigure}[b]{0.49\textwidth}
\centering
\begin{tikzpicture}[scale=3]

\def\step{0.25}

\foreach \i in {0,1,2,3,4}{
    \draw[gray!20, thin] ({\i*\step},0) -- ({\i*\step},1);
    \draw[gray!20, thin] (0,{\i*\step}) -- (1,{\i*\step});
}
\foreach \i in {0,1,2,3,4}{
\foreach \j in {0,1,2,3,4}{
        \fill[gray!70,thin] ({\i*\step},{(\j)*\step}) circle (0.8pt);
}
}
\fill[black] (0,1) circle (1.5pt);
\fill[black] (1,0) circle (1.5pt);

\node[above right] at (1,1) {$z$};

\draw[<->, shift={(0,-0.15)}] (0,0) -- (1,0) node[midway, below] {$\Delta$};
\draw[<->, shift={(-0.15,0)}] (0,0) -- (0,1) node[midway, left] {$\Delta$};

\draw[->, thick] (1,0) -- (1,0.33);
\draw[->, thick] (1,0.33) -- (1,0.66);
\draw[->, thick] (1,0.66) -- (1,1);

\draw[->, thick] (0,1) -- (0.33,1);
\draw[->, thick] (0.33,1) -- (0.66,1);
\draw[->, thick] (0.66,1) -- (1,1);

\node[right, font=\scriptsize] at (1.05, 0.9) {\textbullet\ $\in \mathcal{G}^+$};
\node[right, font=\scriptsize,gray!60] at (1.05, 0.75) {\textbullet $\notin \mathcal{G}^+$};
\end{tikzpicture}
\end{subfigure}
\caption{Analysis of the two cases in the proof of \Cref{lemma: E0}, with $g(x)\ge 0$, for all $x\in \mathcal{G}^\OPT$. In the first, there are two possible $L$-paths which follow the Lipschitz direction, in the second just one.} \label{fig:twocases}
\end{figure}

We now show that the graph is highly connected. Since we known that the optimal point belong to $\cG^\OPT$, we just care about the connectivity of the graph restricted to this node. Formally, we define the subgraph which consists of the vertices $\mathcal{G}^\OPT$ and all edges from $\mathcal{R}(\mathcal{G})$ that have both endpoints in $\mathcal{G}^\OPT$.

\begin{definition}[Subgraph $\mathcal{R}^\OPT$]
\label{def: R opt}
    Let $\mathcal{R}(\mathcal{G})$ be the graph induced by $\mathcal{G}$ as per \Cref{def:graph}. We define $\mathcal{R}^\OPT$ as the subgraph of $\mathcal{R}(\mathcal{G})$ induced by the vertex set $\mathcal{G}^\OPT$.
\end{definition}

Then, we show that the subgraph induced by the vertex set $\mathcal{G}^\OPT$ has a small independence number. 
Intuitively, we focus on $\mathcal{G}^\OPT$ because this set contains the optimum under the clean event $E_0$; the remaining points in $\mathcal{G} \setminus \mathcal{G}^\OPT$ serve only to improve the connectivity between points within $\mathcal{G}^\OPT$ (according to \Cref{def:graph}). 

Intuitively, our proof works as follows.
Consider two points $x, y \in \mathcal{G}^\OPT$ at a distance of at most $\Delta$. By the one-sided Lipschitz property of the objective, these points can be connected via two consecutive axis-aligned segments—one horizontal and one vertical—forming an $L$-shaped path whose orientation depends on the relative positions of $x$ and $y$ (see Figure \ref{fig:twocases}). Given the one-sided Lipschitz property, the objective value decreases by at most $\Delta$ along each segment; consequently, both segments consist of $\mathcal{O}(\Delta)$-suboptimal points. The argument is completed by a covering argument, which shows there are at most $\mathcal{O}(1/\Delta^2)$ such non-connected elements.
Formally,

\begin{lemma}\label{lemma: E0} Under the clean event $E_0$, the graph $\mathcal{R}^\OPT$ has independent number
\begin{equation*}
    \alpha(\mathcal{R}^\OPT) \le \mathcal{O}\left(\frac{1}{\Delta^2}\right).
\end{equation*}

\end{lemma}

\begin{proof}
    Suppose the clean event $E_0$ holds. 
    We split the analysis in two cases: 
    \begin{itemize}
        \item $\exists x\in \mathcal{G}^\OPT : g(x)\ge 0$, and therefore $g(x)\ge 0 \;\forall x\in \mathcal{G}^\OPT$
        \item $\nexists x\in \mathcal{G}^\OPT : g(x)\ge 0$, and therefore $g(x)< 0 \;\forall x\in \mathcal{G}^\OPT$.
    \end{itemize}
    We start by studying the first case. 
    
    Consider a point $x=(x_1,x_2)\in \mathcal{G}^\OPT$, so that $g(x)\hat{F}(x)\ge \widetilde{\OPT}-2\Delta$. Then, by the Lipschitzness of $g(\cdot)$, any point $y$ in
\[
Q_x \coloneqq [x_1,x_1+\Delta]\times [x_2,x_2+\Delta]\cap \mathcal{G}^U(1/\epsilon)
\]
satisfies $g(y)\hat{F}(y)\ge \widetilde{\OPT}-6\Delta$.
Indeed,
\begin{align*}
    \widetilde{\OPT}-2\Delta&\le g(x)\hat{F}(x)\\
    &\le g(x)F(x)+\Delta \\
    &\le g(x)F(y)+\Delta\\
    &\le (g(x)-g(y))F(y)+g(y)F(y)+\Delta\\
    & \le \lVert x-y\rVert_1+ g(y)(F(y)-\hat{F}(y))+g(y)\hat{F}(y)+\Delta\\
    & \le g(y)\hat{F}(y)+4\Delta,
\end{align*}
where we used the definition of the clean event $E_0$, $g(x)\ge0$, $F(y)\le 1$, $F(x)\le F(y)$, the $1$-Lipschitzness of $g$  and $\lVert x-y\rVert_1\le 2\Delta$.
 
We now prove that any two points $x,y\in \mathcal{G}^\OPT$ are connected in $\mathcal{G}$ if $\lVert x-y\rVert_\infty\le \Delta$. There are two possible cases:
\begin{enumerate}
    \item either $x\le y$ or $y\le x$ holds component-wise;
    \item both $x\nleq y$ and $y\nleq x$.
\end{enumerate}
Let us consider the first case and, without loss of generality, assume $x\le y$. Then $y\in Q_x$, and therefore for any $z\in T(x,y)$, both $\mathfrak{I}(z,x)$ and $\mathfrak{I}(z,y)$ are contained in $Q_x$, which means that for every choice of $z\in T(x,y)$ it holds $\tilde{r}(\mathfrak{I}(x,z)\cup \mathfrak{I}(y,z))\ge \widetilde{\OPT}-6\Delta$, and $x$ and $y$ are connected in $\mathcal{G}$.
 
Let us now consider the second case, and, without loss of generality, assume $x_1<y_1$ and $y_2<x_2$. If we define $z\in T(x,y)$ as $z=(y_1,x_2)$, then $\mathfrak{I}(x,z)\in Q_x$ and $\mathfrak{I}(z,y)\in Q_y$, and means that for this choice of $z$ we have $\tilde{r}(\mathfrak{I}(x,z)\cup \mathfrak{I}(y,z))\ge \widetilde{\OPT}-6\Delta$, and $x$ and $y$ are connected in $\mathcal{G}$.
 
Hence, for any pair of points $x,y\in \mathcal{G}^\OPT$ such that $\lVert x-y\rVert_\infty\le \Delta$, there always exists $z \in T(x,y)$ such that $\tilde r(\mathfrak{I}(x,z)\cup \mathfrak{I}(y,z))  \ge \widetilde{\OPT}-6\Delta$.

 %
Therefore, we have proven that any two points in $\mathcal{G}^\OPT$ whose $\ell_\infty$ distance is at most $\Delta$ are connected. The independence number restricted to graph $\mathcal{R}^\OPT$ is therefore upper-bounded by the maximum number of pairwise disjoint $\ell_\infty$-balls of radius $\Delta$ that can simultaneously fit in $[0,1]\times[0,1]$, which is $\mathcal{O}\left(\frac{1}{\Delta^2}\right)$. 

Finally, to conclude the proof we consider the case in which $g(x)<0$ for all $x\in \mathcal{G}^\OPT$. 
The analysis here is specular to the one above, namely, defining for any $x\in \mathcal{G}^\OPT$
\[Q^-_x \coloneqq [x_1-\Delta,x_1]\times [x_2-\Delta,x_2]\cap \mathcal{G}^U(1/\epsilon)\]
we can prove that any $y\in \mathcal{G}^\OPT \cap Q^-_x$ is directly connected to $x$ in $\mathcal{R}(\mathcal{G})$. Indeed, similarly to before

\begin{align*}
    \widetilde{\OPT}-2\Delta&\le g(x)\hat{F}(x)\\
    &\le g(x)F(x)+\Delta \\
    &\le g(x)F(y)+\Delta\\
    & \le g(y)\hat{F}(y)+4\Delta,
\end{align*}
where we used the fact that $F(y)\le F(x)$ and $g(x)<0$.

By adapting the choice of $z\in T(x,y)$ to the new notion of $Q^-_x$, we can retrieve the desired guarantee even in this case, following the same steps above.
\end{proof}

\section{Relative Learning on an Incomplete Grid} \label{sec:relative}

In this section we address the following question: given a grid
$\mathcal{G}\subseteq\mathcal{G}^U(K)$, how can we approximate the distance function
$D(x;y)=F(y)-F(x)$ using as few queries as possible, while restricting all queried
points to $\mathcal{G}$?

When $\mathcal{G}=\mathcal{G}^U(K)$ is the complete grid, this problem is already
solved by \citet{castiglioni2026sample}: estimating $D$ over the full grid is
equivalent to estimating $F$ itself, which costs
$\frac{1}{\epsilon^3}\log(K/\delta\epsilon)^{\mathcal{O}(1)}$ queries.
However, this procedure does not extend to an arbitrary incomplete grid $\mathcal{G}$. Its binary search routines implicitly assume that every point on the path between two
queried locations is itself available to query, which fails as soon as $\mathcal{G}$
contains holes.

We therefore relax the goal: rather than estimating $F$ at every point of
$\mathcal{G}$, we only estimate $D(x;y)$ for pairs $x,y$ that are \emph{directly
connected} in $\mathcal{R}(\mathcal{G})$.
Our procedure is presented in \Cref{alg: RLS}. First, we replace the standard distance-based binary search of \citet{castiglioni2026sample} --- which bisects an interval at its geometric
midpoint ---  with an \emph{index-based} binary search \textsf{BinS} that only bisects at points
already present in $\mathcal{G}$ (\Cref{alg:bins}). We perform this binary search both on the horizontal ($H$) and vertical ($V$) axis. The output of this procedure provides estimation of the probability that $X\sim \cD$ falls in some rectangles (see \Cref{sec:BS}).

Then, we need to use such estimates to approximate the distance between connected points using function \textsf{Estimate-Construction}. We perform this step in two phases. First, we show that any segment
$(a,b] \times c$ (or $c\times [a,b])$  can be decomposed into $O(\log K)$ dyadic sub-intervals for which the
index-based search has already produced estimates, allowing us to assemble
$\widehat{D}(x;y)$ by summing the probability of the $O(\log K)$ added areas (see \Cref{sec: logarithmic comp}). Finally, we extend this distance estimates to general edges by recalling that they are the sum of (at most) two $L$-shaped segments (\Cref{sec:differences}).

\begin{algorithm}[H]
\caption{\textsc{Relative learning sampling (RLS)}}
\label{alg: RLS}
\begin{algorithmic}[1]
\Require grid $\mathcal{G}\subseteq \mathcal{G}^U(K)$, error $\epsilon$

\State $\textsf{BinS}^H(I,\epsilon,\delta/(2|\mathcal{I}^*|),\mathcal{G})$ for all $I\in \mathcal{I}^\star$ \Comment{see \Cref{eq:Istar} for the definition of $\mathcal{I}^\star$}
\State $\textsf{BinS}^V(I,\epsilon,\delta/(2|\mathcal{I}^*|),\mathcal{G})$ for all $I\in \mathcal{I}^\star$
\State\Return $\hat D(\cdot;\cdot)\gets$ \textsc{Estimate-Construction}($\cG$)
\end{algorithmic}
\end{algorithm}



\subsection{From Distance-Based Binary Search to Index-Based Binary Search} \label{sec:BS}

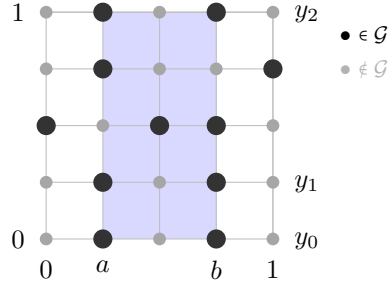
\begin{figure}
\centering

\begin{tikzpicture}[scale=3]

\def\K{4}
\def\step{0.25}

\fill[blue!15] (0.25,0) rectangle (0.75,1);

\foreach \i in {0,1,2,3,4}{
    \draw[gray!50, thin] ({\i*\step},0) -- ({\i*\step},1);
    \draw[gray!50, thin] (0,{\i*\step}) -- (1,{\i*\step});
}
\draw[gray!50, thin] (1,0) -- (1,1);
\draw[gray!50, thin] (0,1) -- (1,1);


\def\subsetpoints{(0.25,0.25),(0.25,0.75),(0.5,0.5),(0.75,0.25),(0.75,1.0),(1.0,0.75)}


\foreach \i in {0,1,2,3,4}{
    \foreach \j in {0,1,2,3,4}{
        \fill[gray!70] ({\i*\step},{\j*\step}) circle (0.8pt);
    }
}

\foreach \pt in {(0.25,0),(0.75,0),(0.5,0.5),(0.25,0.75),(0.25,1),(0.75,0.25),(0.75,1.0),(1.0,0.75), (0.25,0.25), (0.75,0.5), (0,0.5)}{
    \fill[black!80] \pt circle (1.2pt);
}

\node[below] at (0.25,-0.05) {\small $a$};
\node[below] at (0.75,-0.05) {\small $b$};
\node[left] at (-0.05,-0.0) {\small $0$};
\node[left] at (-0.05,1) {\small $1$};
\node[below] at (-0.0,-0.05) {\small $0$};
\node[below] at (1,-0.05) {\small $1$};
\node[below] at (1.15,1+0.075) {\small $y_2$};
\node[below] at (1.15,0.25+0.075) {\small $y_1$};
\node[below] at (1.15,0.075) {\small $y_0$};

\node[right, font=\scriptsize] at (1.25, 0.9) {\textbullet\ $\in \mathcal{G}$};
\node[right, font=\scriptsize, gray!60] at (1.25, 0.75) {\textbullet\ $\notin \mathcal{G}$};

\end{tikzpicture}
\caption{Example of index-based binary search on the column between $a$ and $b$. $y_0$, $y_1$, and $y_2$ denote the three coordinates $y$ for which both $(a,y)$ and $(b,y)$ belong to the grid $\cG$.} \label{fig:bins}
\end{figure}

By symmetry between the two coordinates, we present all procedures along a single axis --- specifically, the vertical-axis --- as extending the procedure to the other axis is straightforward.

The pseudo-code for this procedure is detailed in \Cref{alg:bins}.
 First, fix a column pair $a, b \in G^U(K)$ and let
\[\mathcal{W}=\{w_0<w_1<\cdots<w_M\}\coloneq \left\{w\in G^U(K):(a,w),(b,w)\in\mathcal{G}\right\}\]
be the set of rows where both $(a,\cdot)$ and $(b,\cdot)$ are present in $\mathcal{G}$ (see Figure \ref{fig:bins}). Our goal is to estimate $\mathbb{P}\bigl((a,b]\times[0,w]\bigr)$ for every $w \in \mathcal{W}$ using a number of queries that scales with the total probability mass of $(a,b]\times[0,1]$ rather than with $M$.

The key insight is to perform a binary search over the indices of $\mathcal{W}$ rather than the geometric midpoints of $[0,1]$. Given a range $[w_i, w_j]$, we split at the median index $w_k$ where $k = \lfloor(i+j)/2\rfloor$. By construction, every queried row is guaranteed to belong to $\mathcal{W}$. We then recurse only into sub-ranges that carry non-negligible probability mass. This strategy guarantees that the total number of $\textsc{Estimate}$ calls is governed by the mass of $(a,b]\times[0,1]$ and almost independent of the grid resolution $K$.

\begin{algorithm}[h]
\caption{\textsc{Binary Subdivision on an Incomplete Grid (colums)(\textsc{BinS})}}
\label{alg:bins}
\begin{algorithmic}[1]
    \Require Interval $(a,b]$; accuracy $\epsilon > 0$; confidence $\delta \in (0,1)$; grid $\mathcal{G}\subseteq \mathcal{G}^U(K)$ 
    \Ensure The properties stated in \Cref{thm:binS} are satisfied
    \Statex \hrulefill
    \Function{\textsc{BinS}}{$(a,b]$, $\epsilon$, $\delta$, $\mathcal{G}$}
    \State $\mathcal{W}=\{w_0<w_1<\cdots<w_M\}\coloneq\{w\in G^U(K):(a,w),(b,w)\in\mathcal{G}\}$
    \State $\mathcal{I} \gets \varnothing$
    \If{$\mathcal{W}=\emptyset$} \Comment{Check $\mathcal{W}$ is not empty}
    \State \Return $\mathcal{I}$
    \EndIf
    \State $\widehat{P}((a,b] \times [0,w_0]) \gets \textsc{Estimate}((a,b]\times [0,w_0],\epsilon)$
    \If{$M>0$} \Comment{Check $\mathcal{W}$ contains more than one element}
    \State $\widehat{P}((a,b]\times [0,w_M]) \gets \textsc{Estimate}((a,b]\times [0,w_M],\epsilon)$ \\\Comment{Compute estimates at the endpoints}
    \State $\mathcal{I} \gets $\Call{\textsc{BinS-Rec}}{$(a,b]\times [w_0, w_M], \epsilon, \delta, \mathcal{W}$}
    \EndIf
    \State \Return $\mathcal{I}\cup [0,w_0]$
    \EndFunction
    \Statex \hrulefill
    \Function{\textsc{BinS-Rec}}{$(a,b]\times[w_i, w_j], \epsilon, \delta$, $\mathcal{W}$}
        \If{$\widehat{P}((a,b]\times [0,w_j])-\widehat{P}((a,b]\times [0,w_i])-4 \epsilon \sqrt{\frac{\ln(\nicefrac{4K}{\delta})}{2}} < \epsilon \; \vee \; j=i+1$}
            \State \Return $(w_i,w_j]$
        \EndIf
        \State $k\gets \lfloor (i+j)/2\rfloor$\Comment{Compute mid-point of the interval}
        \State $\widehat{P}((a,b]\times [0,w_k]) \gets \textsc{Estimate}((a,b]\times [0,w_k],\epsilon)$  \label{line:call_Estimate}\Comment{Estimate probability of $(a,b] \times [0,w_k]$}
        \State \Return \Call{\textsf{BinS-Rec}}{$(a,b]\times [w_i, w_k], \epsilon, \delta, \mathcal{W}$} $\cup$ \Call{\textsf{BinS-Rec}}{$(a,b]\times [w_k, w_j], \epsilon, \delta, \mathcal{W}$} \\\Comment{Recursive calls}
    \EndFunction
\end{algorithmic}
\end{algorithm}

Let $\textsc{Extremes}(\mathcal{I})$ denote the union of the endpoints of all
intervals in a collection $\mathcal{I}$. Adapting the analysis of \citep{castiglioni2026sample} to be index-based, we get the following guarantees on the resulting partition.

 \begin{restatable}{lemma}{lemmaBINS}[Adapted from \cite{castiglioni2026sample}] \label{thm:binS}
With probability at least $1-\delta$, the partition $\mathcal{I}$ returned by \Cref{alg:bins} satisfies the following three properties.
\begin{enumerate}
    \item For every $w_i\in \textnormal{\textsc{Extremes}}(\mathcal{I})$ of an interval in $\mathcal{I}$, it holds:
    \[\left| \widehat{P}((a,b]\times [0,w_i])-\mathbb{P}(X\in (a,b]\times [0,w_i])\right|\le \epsilon \, 2\; \sqrt{\frac{\ln(\frac{4K}{\delta})}{2}} .\]
    \item For every interval $I \in \mathcal{I}$ such that $I = (w_i, w_j]$, it holds:
    \[
    \mathbb{P}({X} \in [a,b] \times I) \le \epsilon+\epsilon \, 2^{3} \; \sqrt{\frac{\ln(\frac{4K}{\delta})}{2}} \quad \bigvee \quad w_{i+1}= w_j  
    \]
    \item The cardinality of $\mathcal{I}$ satisfies:
    \[ |\mathcal{I}| \le \frac{2 \mathbb{P}_{X\sim \cD}(X \in (a,b]\times [0,1])\log_2 K}{\epsilon} +2.\]
\end{enumerate}
Moreover, \Cref{alg:bins} uses at most $\frac{1}{\epsilon^2}|\mathcal{I}|$ queries. 
\end{restatable}
Intuitively, the lemma establishes three key properties:
\begin{enumerate}
    \item Every endpoint produced by the recursion carries an $\epsilon$-accurate probability estimate.
    \item Every leaf interval either contains negligible probability mass or corresponds to two adjacent rows of $\mathcal{W}$ that cannot be split further.
    \item The total number of leaves --- and consequently, the total number of $\textsc{Estimate}$ calls --- scales with the probability mass of the column $(a,b]$ rather than the grid resolution $K$.
\end{enumerate}

Because \Cref{alg:bins} only computes estimates at rows within $\textsc{Extremes}(\mathcal{I})$, querying an arbitrary row $w \in \mathcal{W}$ requires interpolating between the two extremes of the interval containing $w$. This interpolation is handled by the \textsc{Compute-Prob} procedure (\Cref{alg:compute-prob}). By \Cref{thm:binS}, the total probability mass fluctuates by at most $\epsilon$ between these two extremes; thus, the interpolated probability is guaranteed to remain $\epsilon$-accurate.

\begin{algorithm}[H]\caption{\textsc{Compute-Prob}} \label{alg:compute-prob}
    \begin{algorithmic}[1]
        \State Input $(a,b]\subseteq [0,1], w\in [0,1]$
        \State $w_2\gets $ minimum $w'\ge w$ such that $\exists \,\hat{P}((a,b]\times [0,w'])$
        \State $w_1\gets $ maximum $w'\le w$ such that $\exists \,\hat{P}((a,b]\times [0,w'])$
        \State\Return $\frac{\hat{P}((a,b]\times [0,w_1])+\hat{P}((a,b]\times [0,w_2])}{2}$
    \end{algorithmic}
\end{algorithm}

\begin{lemma}
\label{lemma: 5.2}
Let $(a,b]$ be the input to \Cref{alg:bins} and $\hat P$ be the computed estimation function.  Then, with probability at least $1-\delta$, for any $w$ such that $(a,w),(b,w)\in \mathcal{G}$, it holds: 
    \begin{align*}
        \left|\mathbb{P}((a,b]\times [0,w])-\frac{\hat{P}((a,b]\times [0,w_1])+\hat{P}((a,b]\times [0,w_2])}{2}\right|\le \frac{\epsilon+\epsilon \, 2^{3} \; \sqrt{\frac{\ln(\frac{4K}{\delta})}{2}}}{2}
    \end{align*}
\end{lemma}

Crucially, \textsc{Compute-Prob} can only be applied to columns $(a,b]$ for which \textsc{BinS} has been executed. Consequently, it remains to determine exactly which columns and rows should undergo the \textsc{BinS} procedure. In the next section, we show that there exists a small set of columns  $\mathcal{I}^\star$  (on which we will run \textsc{Compute-Prob}) such that  every distance between two columns can be decomposed into a logarithmic number of columns in $\mathcal{I}^\star$.

\subsection{Decomposing the Distance in Logarithmic Components}
\label{sec: logarithmic comp}

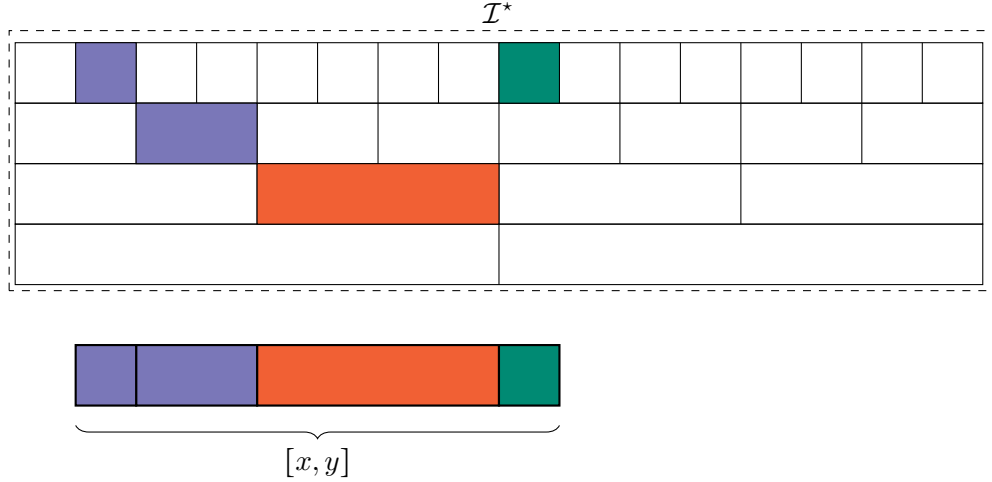
\begin{figure}
    \centering
\begin{tikzpicture}[scale=0.8]
    \draw(0,0) rectangle (16,1);
    \draw(0,1) rectangle (16,2);
    \draw(0,2) rectangle (16,3);
    \draw(0,3) rectangle (16,4);

    \draw[] (8,0)--(8,4);
    
    \draw[] (4,1)--(4,4);
    \draw[] (12,1)--(12,4);

    \draw[] (2,2)--(2,4);
    \draw[] (6,2)--(6,4);
    \draw[] (10,2)--(10,4);
    \draw[] (14,2)--(14,4);

\draw[] (1,3)--(1,4);
    \draw[] (3,3)--(3,4);
    \draw[] (5,3)--(5,4);
    \draw[] (7,3)--(7,4);
    \draw[] (9,3)--(9,4);
    \draw[] (11,3)--(11,4);
    \draw[] (13,3)--(13,4);
    \draw[] (15,3)--(15,4);

    \filldraw[fill= RedOrange] (4,1) rectangle (8,2);
    \filldraw[fill= Periwinkle] (2,2) rectangle (4,3);
    \filldraw[color=black,fill= Periwinkle] (1,3) rectangle (2,4);
    \filldraw[fill= PineGreen] (8,3) rectangle (9,4);

    \draw[] (1,-1) rectangle (9,-2);

    \filldraw[color= black, thick, fill= RedOrange] (4,-1) rectangle (8,-2);
    \filldraw[color= black, thick,fill= Periwinkle] (2,-1) rectangle (4,-2);
    \filldraw[color= black, thick,color=black,fill= Periwinkle] (1,-1) rectangle (2,-2);
    \filldraw[color= black, thick,fill= PineGreen] (8,-1) rectangle (9,-2);

    \draw[dashed] (-0.1,-0.1) rectangle (16.2,4.2); 
\node[] at (8,4.5) {$\mathcal{I}^\star$};
    \draw [decorate,decoration={brace,amplitude=5pt,mirror,raise=4ex}]
  (1,-1.5) -- (9,-1.5) node[midway,yshift=-3em]{$[x,y]$};

\end{tikzpicture}
\caption{Representation of the interval (x,y) in $\mathcal{I}^\star$, the orange interval is the largest one $(a',b']$, the anchor, the blue interval represent the decomposition $(a, a']$ and the green one the decomposition $(b', b]$}
    \label{fig:placeholder}
\end{figure}

While \Cref{thm:binS} provides $\epsilon$-accurate estimates of $\mathbb{P}((a,b]\times[0,w])$, it only applies to column intervals $(a,b]$ for which \textsc{BinS} has explicitly been executed. Running \textsc{BinS} on every possible pair of grid columns is computationally prohibitive. To bypass this bottleneck, we exploit a dyadic structure. A key insight in the analysis of \citet{castiglioni2026sample} is that any interval can be decomposed into $O(\log K)$ dyadic pieces whose sizes are powers of two. We show that a similar decomposition applies to the difference between two arbitrary grid points rather than just the distance from a single point to the origin. This extends the framework of \citet{castiglioni2026sample} --- which is restricted to intervals of the form $[0,a]$ --- to arbitrary intervals $[a,b]$.

We define the set $\mathcal{I}^\star$ on which the procedures bins will be run (both on the horizontal and vertical axis:
\begin{align} \label{eq:Istar}
        \mathcal{I}^\star\coloneq \bigcup_{j\in \left\{0,\ldots,\log_2(K-1)\right\}}\bigcup_{i\in \left\{0,\ldots,\frac{(K-1)}{2^j}-1\right\}}\left\{\bigg(i2^j\frac{1}{K-1},(i+1)2^j\frac{1}{K-1}\bigg]\right\}.
    \end{align}
Notice that $|\mathcal{I}^\star|\le 2K$.\footnote{To account for the interval ${0}$, we also run the procedure on the degenerate interval $(0^-,0]={0}$, where, by convention, $\mathbb{I}(x\le(0^-,w))=0$ deterministically for all $w\in [0,1]$. This modification does not affect the analysis or the order of the guarantees.}
Then, we can prove the following:

\begin{lemma}\label{lemma: I star}
    Let $a,b\in G^U(K)=\{\frac{i}{K-1}: i=0,\ldots,K-1\}$ such that $a<b$. Then Algorithm \ref{alg:computeinterval} on input ($a,b$) returns a collection of at most $2\log_2(K)$ intervals in $\mathcal{I}^\star$ that partition $(a,b]$. 
\end{lemma}

\begin{proof}
The proof mirrors the algorithm's execution and proceeds in two phases: first we
extract a single large sub-interval covering at least a quarter of the target interval,
then we recursively decompose the two leftover boundary segments.

\paragraph{Phase 1: Finding the First Main Sub-Interval.
}
Write $u = (K-1)a$ and $v = (K-1)b$ for the endpoints of $(a,b]$ rescaled to the
integer grid, and let $L = v - u = (K-1)(b-a)$ denote the rescaled length. Let
$2^j$ be the largest power of two not exceeding $L$, i.e., $j = \lfloor \log_2
L\rfloor$. 


To justify the extraction of the first sub-interval, note that because $L \ge 2^{j}$ if $(u,v]$ does not contain a full dyadic interval of size $2^{j}$, the interval $(u,v]$ must contain exactly one multiple of $2^j$, which we denote as $m \cdot 2^j$. This grid point defines two adjacent canonical dyadic intervals of size $2^{j-1}$ in $\mathcal{I}^\star$:
\[
I_{\text{left}} = (m2^j - 2^{j-1}, m2^j] \quad \text{and} \quad I_{\text{right}} = (m2^j, m2^j + 2^{j-1}].
\]
Because the total length $L \ge 2^j = 2 \cdot 2^{j-1}$, then $(u, v]$ must fully enclose at least one of these two canonical intervals. Let $(\frac{n_1}{K-1}, \frac{n_2}{K-1}] \in \mathcal{I}^\star$ be this enclosed interval, where $n_2 - n_1 = 2^{j-1}$ and $u \le n_1 < n_2 \le v$.

\paragraph{Phase 2: Recursive Decomposition of Leftover Intervals.}

After removing the main sub-interval, two boundary segments remain, $(a,
\frac{n_1}{K-1}]$ and $(\frac{n_2}{K-1}, b]$, whose inner endpoints $n_1, n_2$
are already aligned with the dyadic grid. We decompose each recursively; by
symmetry it suffices to treat $(a,\frac{n_1}{K-1}]$.

Let $n^{(0)} = n_1$ be the current right anchor. Since $n^{(0)}$ is, by
construction, a multiple of $2^{j-1}$, it is in particular a multiple of every
smaller power $2^{i}$, $i \le j-1$. We peel off valid sub-intervals from right to
left: at step $k$, given anchor $n^{(k)}$, we do the following:
\begin{enumerate}
    \item Let $2^{i_k}$ be the largest power of two fitting in the remaining
    rescaled gap, i.e., $i_k = \lfloor \log_2\big(n^{(k)}-a(K-1)\big)\rfloor$.
    \item Since, as we observed, $n^{(k)}$ is a multiple of $2^{i_k}$, the point $n^{(k+1)} := n^{(k)} - 2^{i_k}$ is again a
    valid dyadic grid point, and the interval $\big(\tfrac{n^{(k+1)}}{K-1},
    \tfrac{n^{(k)}}{K-1}\big]$ belongs to $\mathcal{I}^\star$.
    \item The remaining gap at least halves at each step:
    \begin{equation*}
    n^{(k+1)} - (K-1)a \;\le\; \tfrac12\big(n^{(k)} - (K-1)a\big).
    \end{equation*}
\end{enumerate}
Iterating, the remaining gap $n^{(k)} - (K-1)a$ shrinks geometrically and
reaches $0$ after at most $\lfloor \log_2((b-a)(K-1))\rfloor$ steps. Hence
$(a,\tfrac{n_1}{K-1}]$ decomposes into at most $\lfloor
\log_2((b-a)(K-1))\rfloor - 1$ elements of $\mathcal{I}^\star$.

By the symmetric argument, $(\tfrac{n_2}{K-1}, b]$ also decomposes into at most
$\lfloor \log_2((b-a)(K-1))\rfloor - 1$ elements of $\mathcal{I}^\star$.

Combining the single interval from Phase 1 with the two recursive
decompositions, $(a,b]$ decomposes into at most
\begin{equation*}
1 + 2\Big(\lfloor \log_2((b-a)(K-1))\rfloor - 1\Big) \;\le\; 2\log_2(K)
\end{equation*}
elements of $\mathcal{I}^\star$, which completes the proof.

\end{proof}

\begin{algorithm}[h]
    \caption{\textsc{ComputeInterval}}\label{alg:computeinterval}
    \begin{algorithmic}[1]
        \Require $a, b \in \left\{ \frac{i}{K-1} \mid i = 0, \ldots, K-1 \right\}: a \leq b$
        \Function{\textsc{ComputeInterval}}{$a,b$}
        \If{$a=b=0$}
        \State \Return $\{0\}$
        \ElsIf{$a=b\neq 0$}
        \State \Return $\mathcal{I}=\emptyset$
        \EndIf
        \State $I \gets $ Largest interval $I = \left( \frac{i \cdot 2^j}{K-1}, \frac{(i+1) \cdot 2^j}{K-1} \right]$ with $i,j \in \mathbb{N}$ such that $I \subseteq [a,b]$
        \State $\mathcal{I} \gets \{ I \}$ \Comment{$\mathcal{I}$ is the collection of all the intervals computed by the procedure}
        \State Let $a', b' \in G^U(K)=\left\{ \frac{i}{K-1} \mid i = 0, \ldots, K-1 \right\}$ be the extremes of $I$, so that $I = [a',b']$
        \While{$a' > a$} \Comment{Cover the sub-interval $[a,a')$ on the left of $I$}
            \State $j \gets \max \left\{ j \in \mathbb{N} \mid a' - \frac{2^j}{K-1}  \geq a \right\}$
            \State $\mathcal{I} \gets \mathcal{I} \cup \left\{ \left(a' - \frac{2^j}{K-1}, a' \right] \right\}$
            \State $a' \gets a' - \frac{2^j}{K-1}$
        \EndWhile
        \While{$b' < b$} \Comment{Cover the sub-interval $(b,b']$ on the right of $I$}
            \State $j \gets \max \left\{ j \in \mathbb{N} \mid b' + \frac{2^j}{K-1}  \leq b \right\}$
            \State $\mathcal{I} \gets \mathcal{I} \cup \left\{ \left(b', b' + \frac{2^j}{K-1} \right] \right\}$
            \State $b' \gets b' + \frac{2^j}{K-1}$
        \EndWhile
        \State \Return $\mathcal{I}$
        \EndFunction
    \end{algorithmic}
\end{algorithm}

\subsection{Reconstruct the Differences From Sampling} \label{sec:differences}

We now show how to combine \textsc{BinS} and \textsc{ComputeInterval} to
construct the distance function returned by \textsc{RLS}, whose pseudocode is
given in \Cref{alg: RLS}. This construction is carried out by the procedure
\textsc{EstimateConstruction} in \Cref{alg: estimate_construction}.

Fix two points $x = (x_1, x_2)$ and $y = (y_1, y_2)$ that are directly connected
in $\mathcal{R}(\mathcal{G})$, and let $z \in T(x,y)$ be the corresponding
corner point (\Cref{eq:corner}). We decompose the distance between $x$
and $y$ along the two axes, reducing the problem to estimating the CDF differences of a
single one-dimensional segment on each axis. By \Cref{lemma: I star}, each such
segment can be decomposed into at most $2\log_2 K$ intervals of
$\mathcal{I}^\star$, each of which is passed to \textsc{BinS} in
\Cref{alg: RLS}. We estimate the probability mass of each interval via
\textsc{Compute-Prob} and sum the resulting estimates to obtain an estimate of
the segment's length; the overall distance estimate $\hat{D}(x;y)$ is then the
sum of the two per-axis estimates. Overall, we get the following guarantees for \textsc{EstimateConstruction} and hence \textsc{RLS}:

\begin{lemma} \label{lemma: dif}
\Cref{alg: RLS} uses at most $\tilde{\mathcal{O}}\left(\dfrac{\log(K)}{\epsilon^3}+\dfrac{K}{\epsilon^2}\right)$
samples, and with probability at least $1-\delta$ outputs, for every pair of points $x, y \in \mathcal{G}$ which are connected in $\mathcal{R}(\cG)$, an estimate $\hat{D}(x;y)$
satisfying
\begin{equation*}
\left|\hat{D}(x;y) - D(x;y)\right| \le \tilde{\mathcal{O}}\left(\epsilon\log_2(K/\delta)\right).
\end{equation*}
\end{lemma}

\begin{proof}

We prove separately the sample complexity bound and the approximation guarantees. In the following, we assume that  \Cref{thm:binS} holds, and hence our results will hold simultaneously for all call of the \textsf{BinS} procedure, with probability at least $1-\delta$, taking the union bound over all $2|\mathcal{I}^\star|$ events, each of probability at least $1-\frac{\delta}{2|\mathcal{I}^\star|}$.

\paragraph{Sample Complexity.}

\Cref{alg: RLS} invokes
$\textsf{BinS}^H(I,\epsilon,\delta/(2|\mathcal{I}^\star|),\mathcal{G})$ and
$\textsf{BinS}^V(I,\epsilon,\delta/(2|\mathcal{I}^\star|),\mathcal{G})$
for every interval $I \in \mathcal{I}^\star$. By \Cref{thm:binS}, each
invocation requires at most
\[
\frac{2\mathbb{P}(X \in I \times [0,1])\log_2 K}{\epsilon^3} + \frac{2}{\epsilon^2}
\]
samples. Since \textsc{Estimate-Construction}$(\mathcal G)$ performs no
additional sampling, it suffices to bound 
\[\sum_{I \in \mathcal{I}^\star}
\mathbb{P}(X \in I \times [0,1]).\]

The family $\mathcal{I}^\star$ is organized into $\log_2(K-1)+1$ dyadic
levels, and at every level the intervals partition $(0,1]$, so
\begin{align*}
\sum_{I \in \mathcal{I}^\star} \mathbb{P}(X \in I \times [0,1])
&= \sum_{j=0}^{\log_2(K-1)}
\left(\sum_{i=0}^{\frac{K-1}{2^j}-1}
\mathbb{P}\!\left(X \in \left(\tfrac{i  2^j}{K-1}, \tfrac{(i+1)2^j}{K-1}\right]
\times [0,1]\right)\right)
+ \mathbb{P}(X \in \{0\} \times [0,1]) \\
&= \sum_{j=0}^{\log_2(K-1)} \mathbb{P}(X \in [0,1]\times[0,1])
= \log_2(K-1) + 1 = \mathcal{O}(\log_2 K),
\end{align*}
where the second equality holds because, at each fixed level $j$, the
intervals $\left(\tfrac{i  2^j}{K-1}, \tfrac{(i+1)2^j}{K-1}\right]$
partition $(0,1]$ exactly.

Substituting this bound into the per-interval sample count and summing over the $\mathcal{O}(\log K)$ calls to \textsf{BinS} at each level, the total
number of samples used by all calls to \textsf{BinS} is
\[
\sum_{I \in \mathcal{I}^*}\left(\frac{2\mathbb{P}(X \in I \times [0,1])\log_2 K}{\epsilon^3} + \frac{2}{\epsilon^2}\right) = \tilde{\mathcal O}\left(\frac{\log_2 K}{\epsilon^3}
+ \frac{K}{\epsilon^2}\right).
\]

\paragraph{Approximation Guarantees}

It remains to show that, with probability at least $1-\delta$, every pair of
connected vertices $x, y \in \mathcal{G}$ satisfies
\[
\left|\hat D(x,y) - D(x,y)\right|
\le \tilde{\mathcal O}\!\left(\epsilon \log_2(K/\delta)\right).
\]

Assume first that $x$ and $y$ are connected through a point $z \in T(x,y)$
sharing the first coordinate with $x$ and the second coordinate with $y$,
namely $z = (x_1, y_2)$, with $x_1 \le y_1$ and $x_2 \le y_2$. Then, it holds:
\begin{align*}
D(x;y)
&= \mathbb{P}(X \le y) - \mathbb{P}(X \le x) \\
&= \bigl(\mathbb{P}(X \le y) - \mathbb{P}(X \le z)\bigr)
+ \bigl(\mathbb{P}(X \le z) - \mathbb{P}(X \le x)\bigr) \\
&= \mathbb{P}\!\left(X \in (x_1,y_1] \times [0,y_2]\right)
+ \mathbb{P}\!\left(X \in [0,x_1] \times (x_2,y_2]\right).
\end{align*}

Let $H$ and $V$ be the collections of intervals constructed by
\Cref{alg: estimate_construction}. By \Cref{lemma: I star}, their elements
are pairwise disjoint, $|H|, |V| \le 2\log_2 K$, and
\[
\bigcup_{I \in H} I = (x_1,y_1], \qquad \bigcup_{I \in V} I = (x_2,y_2].
\]
Hence
\[
\mathbb{P}\!\left(X \in (x_1,y_1] \times [0,y_2]\right)
= \sum_{I \in H} \mathbb{P}(X \in I \times [0,y_2]),
\]
and similarly
\[
\mathbb{P}\!\left(X \in [0,x_1] \times (x_2,y_2]\right)
= \sum_{I \in V} \mathbb{P}(X \in [0,x_1] \times I).
\]

By \Cref{lemma: 5.2}, and the union bound over all calls of \textsf{BinS}, with probability at least $1-\delta$ each term in
these two sums is estimated with additive error at most
\[
8\log_2(K)\,\epsilon\sqrt{\frac{\ln(8K|\mathcal{I}^\star|/\delta)}{2}}.
\]

Summing over the at most $|H| + |V| \le 4\log_2 K$ terms and applying the
triangle inequality, we obtain
\begin{align*}
|\hat D(x;y)-D(x;y)|
&\le
\sum_{I\in H}
\left|
\widehat{P}(X\in I\times[0,y_2])
-
\mathbb P(X\in I\times[0,y_2])
\right|\\
&\quad+
\sum_{I\in V}
\left|
\widehat{ P}(X\in[0,x_1]\times I)
-
\mathbb P(X\in[0,x_1]\times I)
\right|\\
&\le
\widetilde{\mathcal O}\!\left(
\epsilon\log_2(K/\delta)
\right),
\end{align*}
where the polylogarithmic factors absorb the cardinalities of $H$ and $V$.

The remaining configurations of $x$, $y$, and $z$ are handled analogously,
by exchanging the roles of the endpoints and, when necessary, of the
horizontal and vertical components.

\end{proof}

\begin{algorithm}[H]\caption{\textsc{Estimate-Construction}}\label{alg: estimate_construction}
    \begin{algorithmic}[1]
    \Require graph $\mathcal{G}$
        \For{$x=(x_1,x_2),y=(y_1,y_2)\in \mathcal{G}$ directly connected}
        \State $z\gets $ any $z'\in T(x,y)$ s.t. $\mathfrak{I}(x,z')\cup \mathfrak{I}(y,z')\in \mathcal{G}$
        \State $H\gets$ \textsf{Compute-Interval} $((x_1,y_1])$
        \State $V\gets$ \textsf{Compute-Interval}$((x_2,y_2])$
        \If{$z=(x_1,y_2)$}
        \State $\hat{D}(x;z)\gets (2\mathbb{I}(x_2<y_2)-1)\sum_{I\in V}$\textsf{Compute-Prob}$^V(I, x_1)$
        \State $\hat{D}(z;y)\gets (2\mathbb{I}(x_1<y_1)-1)\sum_{I\in H}$\textsf{Compute-Prob}$^H(I, y_2)$
        \ElsIf{$z=(y_1,x_2)$}
        \State $\hat{D}(x;z)\gets (2\mathbb{I}(x_1<y_1)-1)\sum_{I\in H}$\textsf{Compute-Prob}$^H(I, x_2)$
        \State $\hat{D}(z;y)\gets (2\mathbb{I}(x_2<y_2)-1)\sum_{I\in V}$\textsf{Compute-Prob}$^V(I, y_2)$
        \EndIf
        \State $\hat{D}(x;y)\gets \hat{D}(x;z)+ \hat{D}(z;y)$
        \EndFor
    \end{algorithmic}
\end{algorithm}

\section{Share the Feedback}

Here, we address the last challenge in designing our algorithm: learning how
to share feedback among points of the grid.

This resembles a bandit problem with a feedback graph. From the feedback generated by choosing an
arm $x$, it is possible to recover an \emph{approximation} of the reward of a
connected arm $y$, for any edge $(x,y)$ in the graph. We consider graphs equipped with all self-loops, namely the edge $(x,x)$ belongs to the graph for all vertexes $x$.
Unfortunately, this does not directly fit the standard bandits-with-feedback-graphs
framework, since the reward of $y$ is only recovered up to
approximation.

The section is organized as follows. We first define a general model, which
we call \emph{multi-armed bandits with feedback graph and
$\epsilon$-misspecified rewards}, and provide a UCB1-style algorithm showing
how the $\epsilon$-error propagates into the regret bound. We then show how to
adapt this model to our setting, where the feedback graph is the key device
that replaces the dependence of the regret bound on the number of arms with a
dependence on the graph's independence number.

\subsection{Multi-Armed Bandits With  Feedback Graph and $\epsilon$-Misspecified Rewards}

At each round $t \in [T]$, the learner selects an arm $i$ from a set
$\mathcal{M}$ of size $M$, and observes feedback $o_{t,i}\in [0,1]$ and gets reward $r_{t,i} \in [0,1]$. For
each arm, the rewards are i.i.d.\ across rounds according to a fixed,
unknown distribution $D_i$ with mean $\mu_i$.

Unlike the standard bandit problem, the learner has access to additional
feedback according to a graph. In the standard bandits-with-feedback-graph
model, pulling an arm reveals not only the reward of that arm, but also the
rewards of a prescribed set of neighboring arms, as determined by a feedback
graph. Here, we study a generalization in which these neighboring rewards may
be slightly misspecified.

Let $\mathcal{R} = (\mathcal{M}, \mathcal{E})$ be an undirected graph over the
set of arms $\mathcal{M}$, and let $\mathcal{N}(i) = \{j \in \mathcal{M} :
(i,j) \in \mathcal{E}\}$ denote the neighborhood of arm $i$. For every edge
$(i,j) \in \mathcal{E}$, there exists an \emph{unknown} function
$f_{i,j} : [0,1] \to [0,1]$,  such that
\begin{align*}
    \mu_j= \mathbb{E}[f_{i,j}(o_{t,i})]
\end{align*}
The learner does not know $f_{i,j}$, but is given an estimate $\hat{f}_{i,j} : [0,1] \to [0,1]$ satisfying the uniform approximation guarantee
\begin{equation}
    \sup_{x \in [0,1]} |\hat{f}_{i,j}(x) - f_{i,j}(x)| \le \epsilon.
\end{equation}

In summary, at each round the learner receives feedback
only for the selected arm; nevertheless, by combining the graph structure with
the $\epsilon$-accurate estimates $\{\hat{f}_{i,j}\}$, it can form
approximations of the expected rewards of neighboring arms.

For the remainder of this section, we let $\mathbb{I}_t(i)$ denote the
indicator function equal to $1$ if and only if the learner selects arm $i \in
\mathcal{M}$ at round $t$, and we let $n_t(i) = \sum_{\tau=1}^t \mathbb{I}_\tau(i)$.

\subsection{A UCB-Like Approach}

Our analysis is largely inspired by standard tools used in bandits with
feedback graphs. The crucial observation is that, while playing arm $i$,
observing $o_{t,i}$ together with knowing the function $\hat{f}_{i,j}(\cdot)$ suffices to recover a
slightly \emph{biased} estimator of the reward of any arm $j \in
\mathcal{N}(i)$, with bias bounded as
\[
|\mu_j - \mathbb{E}[f_{i,j}(o_{t,i})]| \le \epsilon.
\]

We then follow standard arguments, carrying this bias term through the
analysis, where it appears linearly in the final regret bound.
\Cref{alg: graph} gives the pseudocode, while the following theorem, whose proof is deferred to Appendix \ref{app:graph feedback}, states the
resulting regret guarantee.

\begin{algorithm}[h]\caption{UCB-like for Multi-armed bandits with $\epsilon$-precise structure}\label{alg: graph}
    \begin{algorithmic}[1]
        \State Input: $\mathcal{R}=(\mathcal{M},\mathcal{E})$, $T$, $\delta$, $\epsilon$
        \State  $n_0(i)\gets 0\quad \forall i\in \mathcal{M}$
        \State
        $\overline{r}_{0,i}\gets 1\quad \forall i\in \mathcal{M}$
        \For{$t=1,\ldots,T$}
        \State $i_t\gets \arg \max_{i\in \mathcal{M}} \overline{r}_{t-1}(i)$
        \State Choose $i_t$ and observe $o_{t,i_t}$
        \State $n_t(i_t)\gets n_{t-1}(i_t)+1,\; n_t(i)\gets n_{t-1}(i) \;\forall i\in \mathcal{M}\backslash\{i_t\}$
        \State $\hat{r}_{t,i}\gets \frac{1}{\sum_{j\in \mathcal{N}(i)} n_t(j) }\sum_{j\in \mathcal{N}(i)}\sum_{\tau=1}^t \mathbb{I}_\tau(j)\hat{f}_{j,i}(o_{\tau,j}), \quad \forall i\in \mathcal{M}$ 
        \State $\overline{r}_{t,i} \gets \hat{r}_{t,i}+6\sqrt{\frac{\ln(2TM\delta^{-1})}{\sum_{j\in \mathcal{N}(i)}n_t(j)}}+\epsilon, \quad \forall i\in \mathcal{M}$
        \EndFor
    \end{algorithmic}
\end{algorithm}

\begin{restatable}{theorem}{theoremGraph} \label{thm: graph reg}
    With probability at least $1-\delta$ \Cref{alg: graph} attains 
    \begin{align*}
        R_T= \tilde{\mathcal{O}}\left(\log(M\delta^{-1})\left(\sqrt{\alpha(\mathcal{R})T}+ T\epsilon\right)\right),
    \end{align*}
    where $\alpha(\mathcal{R})$ is the independence number of the graph $\mathcal{R}$.
\end{restatable}

As a direct corollary, we obtain an analogous guarantee for our setting.

\begin{corollary}\label{cor:graph}
    There exists an algorithm that takes in input the graph $\mathcal{R}^\OPT$ (\Cref{def: R opt}) and the estimates $D(x;y)$ for each $x$ and $y$ which are connected by an edge in the graph $\mathcal{R}^{\OPT}$, and under the clean event $E_0$ and the event described by \Cref{lemma: dif} it guarantees regret over $T$ rounds of at most
    \[\widetilde{\mathcal{O}}\left(\log(\delta^{-1})\left(\sqrt{\alpha(\mathcal{R}^\OPT)T}+T\epsilon\right)\right),\]
    with probability at least $1-\delta$.
\end{corollary}

\begin{proof}
The proof proceeds by reducing the problem to an instance of a multi-armed
bandit with graph feedback and $\epsilon$-misspecified rewards, as defined
above.

Fix two connected points $x, y \in \mathcal{G}^{\mathrm{OPT}}$. Recall that
$\mu_{x} = g(x)F(x)$, $\mu_{y} = g(y)F(y)$, and $F(y) = F(x) + D(x;y)$. Notice that we can recover an instance with positive rewards by a simple normalization. 
Additionally, by choosing $x$ at round $t$ we observe $o_{t,x}=\mathbb{I}(X_t\le x)$.
Define
$
f_{x,y}$ as
\[
f_{x,y}(o) = \left(o + D(x;y)\right) g(y).
\]
Since $f_{x,y}$ is affine, it holds
\begin{align*}
\mathbb{E}[f_{x,y}(o_{t,x})] &=f_{x,y}(\mathbb{E}[\mathbb{I}(X_t\le x)]) \\
&= f_{x,y}(F(x))\\
&= \left(F(x) + D(x;y)\right) g(y) \\
&
= F(y)\, g(y) = \mu_y,
\end{align*}
so that $f_{x,y}$ satisfies the exact-reward-recovery condition required by
the reduction. Notice that if $x=y$ then $D(x;y)=0$ and $f_{x,x}(o)=o \cdot g(x)$.

Similarly, define the estimated map
\[
\hat{f}_{x,y}(o) = \left(o + \hat{D}(x,y)\right) g(y).
\]
Then, we get that for each $o\in [0,1]$, it holds $|\hat{f}_{x,y}(o)|\le \mathcal{O}(1)$.

Moreover, assuming the event of \Cref{lemma: dif},
\begin{align*}
|f_{x,y}(o) - \hat{f}_{x,y}(o)|
&= \bigl|(D(x;y) - \hat{D}(x;y))\, g(y)\bigr|
\\&\le |D(x;y) - \hat{D}(x;y)|
\le \tilde{\mathcal{O}}(\epsilon \log_2(\delta^{-1})) \qquad \forall o \in [0,1],
\end{align*}
where we used $K=1/\epsilon$. Finally, for this choice of the set $\mathcal{M}$ we have that $M=|\mathcal{G}^\OPT|\le K^2$.
This concludes the proof.
\end{proof}

\section{Putting Everything Together}

We are ready to prove the main theorem, which we restate here for clarity.

\thmmain*

\begin{proof}
Let $x^*$ be defined as the optimum on the grid $\mathcal{G}^U(1/\epsilon)$, i.e. $x^*\in \arg \max_{x\in \mathcal{G}^U(1/\epsilon)}\sum_{t=1}^Tg(x)F(x)$, and let us consider \Cref{alg: main non-bounded}. Let $P_1,P_2,P_3$ be the set or rounds in which the algorithm is playing Phase $1 ,2, 3$ respectively (the \textsf{BuildGraph} procedure is only a computational procedure that does not require interaction with the environment and therefore does not generates regret).
    To prove the theorem it is necessary to decompose the regret in 4 components: 
    \begin{itemize}
        \item The approximation error : $\max_{x\in [0,1]^2}\sum_{t=1}^T g(x)F(x)-\max_{x\in \mathcal{G}^U(1/\epsilon)}\sum_{t=1}^T g(x)F(x)$
        \item The regret of the first phase: $R_T^1= \sum_{t\in P_1}g(x^*)F(x^*)-\sum_{t\in P_1}g(x_t)F(x_t)$
        \item The regret of the second phase: $R_T^2= \sum_{t\in P_2}g(x^*)F(x^*)-\sum_{t\in P_2}g(x_t)F(x_t)$
        \item The regret of the third phase: $R_T^3= \sum_{t\in P_3}g(x^*)F(x^*)-\sum_{t\in P_3}g(x_t)F(x_t)$
    \end{itemize}

    \paragraph{Bounding $R_T^1$.}
    The regret suffered in the first phase can be upperbounded by the number of rounds in $P_1$. 
    Hence, by \Cref{theo: mainOLD}
    \[R_T^1\le |P_1|\le \mathcal{O}\left(\frac{1}{\Delta^3}\log(1/\epsilon\delta)^{\mathcal{O}(1)}\right).\]

    \paragraph{Bounding $R_T^2$.}
    The regret suffered in the second phase can be upperbounded by the number of rounds in $P_2$ multiplied for the highest possible sub-optimality gap in $\mathcal{G}$, as this phase is specifically designed to choose only elements in $\mathcal{G}$.
    
    We can bound the highest sub-optimality gap in $\mathcal{G}$ under the event $E_0$ with $8\Delta$ thanks to \Cref{lemma: subopt gap}, and we can bound $|P_2|$ using the result in terms of sample complexity of \Cref{lemma: dif} 
    
    In particular, recall that the parameter $K$, which denotes the size of the one-dimensional discretization, is set to $K=1/\epsilon$ in \Cref{alg: main non-bounded}; that is, the initial uniform grid is $\mathcal{G}^U(1/\epsilon)$. 
    
    Hence, under the event $E_0$, 
    \[R_T^2\le 8\Delta|P_2|\le \mathcal{O}\left(\left(\frac{\Delta}{\epsilon^3}\right)\log(1/\epsilon)\right).\]

    \paragraph{Bounding $R_T^3$.}
    The regret suffered by the third phase can be bounded using \Cref{cor:graph} and \Cref{lemma: E0}, so that under $E_0$ and with probability at least $1-\delta$, 
    \[R_T^3\le \widetilde{\mathcal{O}}\left(\frac{1}{\Delta}\sqrt{T\ln(1/\epsilon\delta)}+T\epsilon\right).\]

    \paragraph{Bounding the Discretizations Error.}
The difference between the optimum on the uniform grid and the optimum on the continuum can be bounded by the 1-Lipschitzness of $g(\cdot)$. Indeed, for any $x\in [0,1]^2$, it exists a $y\in \mathcal{G}$, $y\ge x$ such that $\lVert y-x\rVert_1\le 2\epsilon$, so that $g(y)F(y)\ge g(y)F(x)\ge g(x)F(x) - 2\epsilon$. Hence,
\[\max_{x\in [0,1]^2}\sum_{t=1}^T g(x)F(x)-\max_{x\in \mathcal{G}^U(1/\epsilon)}\sum_{t=1}^T g(x)F(x)\le 2T\epsilon.\]

To conclude, by setting $\Delta= T^{-2/10}$, $\epsilon= T^{-3/10}$ and $\delta=\frac{1}{T}$
    \begin{align*}
        R_T&= \tilde{\mathcal{O}}\left(\frac{1}{\Delta^3}\log(1/\epsilon\delta)^{\mathcal{O}(1)}+\left(\frac{\Delta }{\epsilon^3}\right)\log(1/\epsilon)+\frac{1}{\Delta}\sqrt{T\ln(1/\epsilon\delta)}+T\epsilon+T\delta\right) \le \tilde{\mathcal{O}}\left(T^{7/10}\right),
    \end{align*}
    which concludes the proof.
\end{proof}

\newpage
\printbibliography

\newpage

\appendix

\section{Omitted Proofs from \Cref{sec:relative}}
For completeness, we report the pseudocode of \textsc{Estimate} (called \textsc{Monte Carlo Estimation} in \citep{castiglioni2026sample}).

\begin{algorithm}[!h]
\caption{\textsc{Monte Carlo Estimation ({Estimate})}}\label{alg: Estimate}
\begin{algorithmic}[1]
    \Require Interval $(a,b]$; threshold $w \in [0,1]$; accuracy $\epsilon > 0$
    \Ensure The value of $\widehat P((a,b]\times[0,w])$ is an estimate of $\mathbb{P}( (a,b] \times [0,w])$
    %
    \Statex \hrulefill
    \Function{Estimate}{$(a,b]\times[0,w],\epsilon$}
    \State $N \gets \lceil 1/\epsilon^2 \rceil$ 
    \For{$\tau = 1, \ldots, N$}
        \State Sample $y_\tau $ uniformly at random from $\{0,1\}$\label{line:sample}
        \State Let $x_\tau \in [0,1]^2: x_{\tau} = \left\{ \begin{array}{ll}
             (a,w), & \text{with probability  } 1/2 \\
             (b,w) & \text{with probability  } 1/2 
        \end{array}\right.$ 
        \State Query $x_\tau$ and observe feedback $ \mathbb{I}[X_{\tau}\le x_{\tau}]$
        \State Compute $z_\tau$ by using the inclusion–exclusion principle: $z_\tau \gets 2 \cdot (-1)^{\mathbb{I}(x_\tau=(a,w))} \cdot \mathbb{I}[X_{\tau}\le x_{\tau}]$ 
    \EndFor
    \State $\widehat P((a,b]\times [0,w])\gets \frac{1}{N}\sum_{\tau \in [N]} z_\tau$
    \EndFunction
\end{algorithmic}
\end{algorithm}

Moreover, we restate their result:
\lemmaCastiglioni*

Then, we are ready to proof \Cref{thm:binS}.

\lemmaBINS*

\begin{proof}

The proof is substantially analogous to the one provided for a complete grid $\mathcal{G}^U(K)$.

The procedure \textsc{Estimate} is called deterministically at most once for every element in $\mathcal{W}$, and $|\mathcal{W}|\le K $ by construction.
Hence, we define the clean event \( \mathcal{E} \) under which every call to \textsc{Estimate} made by \textsc{BinS} satisfies the property in \Cref{lem:Estimate} for a confidence parameter \( \delta' \coloneq \delta / (K) \). Given that $\mathcal{E}$ is simply the intersection of at most $K$ events, by a union bound \( \mathcal{E} \) holds with probability at least $1 -\delta$.

Under the clean event $\mathcal{E}$ the following holds:
\[\left\lvert \widehat{P}((a,b]\times [0,w])-\mathbb{P}_{X\sim \cD}(X\in (a,b]\times[0,w])\right\rvert\le \epsilon\,2\;\sqrt{\frac{\ln(\nicefrac{4K}{\delta})}{2}} \quad \forall w\in \textnormal{\textsc{Extremes}}(\mathcal{I}),\]
which is exactly the first property in the statement of the lemma.


Then, for all $I \in \mathcal{I}$ such that $I = (w_i,w_{i+1}]$, under the event $\mathcal{E}$ it holds:

\begin{align*}
    \widehat{P}([a,b] \times [0,w_{i+1}]) - \widehat{P}&([a,b]\times [0,w_{i}]) - \epsilon\,2\left(2\sqrt{\frac{\ln(\nicefrac{4K}{\delta})}{2}}\right) \\
    & \ge \mathbb{P}_{X\sim \cD}\left(X \in [a,b] \times (w_i,w_{i+1}]\right)- \epsilon\,4\left(2\sqrt{\frac{\ln(\nicefrac{4K}{\delta})}{2}}\right)\label{eq: phat difetto},
\end{align*}
and 
\begin{align*}
    \widehat{P}([a,b]\times [0,w_{i+1}]) - \widehat{P}([a,b]\times [0,w_{i}]) - \epsilon\,2\left(2\sqrt{\frac{\ln(\nicefrac{4K}{\delta})}{2}}\right)\le \mathbb{P}_{X\sim \cD}(X \in [a,b] \times (w_i,w_{i+1}])\label{eq: phat eccesso}.
\end{align*}
This proves the second property in the statement of the lemma. Indeed
\begin{align*}
    \mathbb{P}_{X\sim \cD}(X \in [a,b] \times (w_i,w_{i+1}])\le \epsilon + \epsilon\,8\sqrt{\frac{\ln(\nicefrac{4K}{\delta})}{2}}.
\end{align*}

Finally, the splitting condition on the probability is identical to the one of \cite{castiglioni2026sample}, and with analogous analysis, letting $n_k \in \mathbb{N}$ be the number of sub-intervals $(y_i,y_j]$ such that $\lfloor j-i \rfloor=2^{-k}$, under $\mathcal{E}$, we have for each splitted $(y_i,y_j]$:
\[
\mathbb{P}_{X\sim \cD}(X \in [a,b] \times (y_i, y_j]) \ge  \widehat{P}([a,b] \times [0, y_j]) - \widehat{P}([a,b] \times [0, y_i]) -\epsilon\,2\left(2\sqrt{\frac{\ln(4K/\delta)}{2}}\right) \ge \epsilon.
\]
Hence, 
\[
n_k \, \epsilon \le \mathbb{P}_{X\sim \cD}(X \in [a,b]\times [0,1]).
\]
To conclude, each split introduces one new interval in $\mathcal{I}$, and each call to \textsf{Estimate} employs $1/\epsilon^2$ samples, therefore
\[
|\mathcal{I}| 
   \le \sum_{k=0}^{\log_2 K} n_k +2
   \le (\log_2 K+1)\,\frac{\mathbb{P}_{X\sim \cD_{j}}(X \in [a,b]\times [0,1])}{\epsilon} +2,
\]
where the plus two accounts for the case $\mathbb{P}_{X\sim \cD_{j}}(X \in [a,b]\times [0,1])=0$.
\end{proof}

\section{Proof of \Cref{thm: graph reg}} \label{app:graph feedback}

We start by giving some preliminary results. 

\begin{lemma}
    Set $t\in [T], i\in \mathcal{M}$. 
    Then, with probability at least $1-\delta$
    \begin{equation*}
      \big|\hat{r}_{t,i}-\mathbb{E}[r_{t,i}]\big|\le 6\sqrt{\frac{\ln(\delta^{-1})}{\sum_{j\in \mathcal{N}(i)}n_t(j)}}+ \epsilon  
    \end{equation*}
\end{lemma}
\begin{proof}
As a first step, we relate the expected value of the estimator $\hat{r}_{t,i}$ and the expected reward of arm $i$.

\begin{align*}
   \mathbb{E}[\hat{r}_{t,i}]& =  \mathbb{E}\left[\frac{1}{\sum_{j\in \mathcal{N}(i)} n_t(j) }\sum_{j\in \mathcal{N}(i)}\sum_{\tau=1}^t \mathbb{I}_\tau(j)\hat{f}_{j,i}(o_{\tau,j})\right]\\
   & = \mathbb{E}\left[\frac{1}{\sum_{j\in \mathcal{N}(i)} n_t(j) }\sum_{j\in \mathcal{N}(i)}\sum_{\tau=1}^t \mathbb{I}_\tau(j){f}_{j,i}(o_{\tau,j})\right]\\
   & + \mathbb{E}\left[\frac{1}{\sum_{j\in \mathcal{N}(i)} n_t(j) }\sum_{j\in \mathcal{N}(i)}\sum_{\tau=1}^t \mathbb{I}_\tau(j)\bigg(\hat{f}_{j,i}(o_{\tau,j})-f_{j,i}(o_{\tau,j})\bigg)\right]\\
   & =\mathbb{E}\left[\mathbb{E}\left[\frac{1}{\sum_{j\in \mathcal{N}(i)} n_t(j) }\sum_{j\in \mathcal{N}(i)}\sum_{\tau=1}^t \mathbb{I}_\tau(j){f}_{j,i}(o_{\tau,j})\bigg|\{\mathbb{I}_\tau(j)\}_{\tau,j}\right]\right]\\
   & + \mathbb{E}\left[\frac{1}{\sum_{j\in \mathcal{N}(i)} n_t(j) }\sum_{j\in \mathcal{N}(i)}\sum_{\tau=1}^t \mathbb{I}_\tau(j)\bigg(\hat{f}_{j,i}(o_{\tau,j})-f_{j,i}(o_{\tau,j})\bigg)\right]\\
   & = \mathbb{E}\left[\frac{1}{\sum_{j\in \mathcal{N}(i)} n_t(j) }\sum_{j\in \mathcal{N}(i)}\sum_{\tau=1}^t \mathbb{I}_\tau(j)\mathbb{E}[r_{t,i}]\right]\\
   & + \mathbb{E}\left[\frac{1}{\sum_{j\in \mathcal{N}(i)} n_t(j) }\sum_{j\in \mathcal{N}(i)}\sum_{\tau=1}^t \mathbb{I}_\tau(j)\bigg(\hat{f}_{j,i}(o_{\tau,j})-f_{j,i}(o_{\tau,j})\bigg)\right]\\
   & =\mathbb{E}[r_{t,i}]+ \mathbb{E}\left[\frac{1}{\sum_{j\in \mathcal{N}(i)} n_t(j) }\sum_{j\in \mathcal{N}(i)}\sum_{\tau=1}^t \mathbb{I}_\tau(j)\bigg(\hat{f}_{j,i}(o_{\tau,j})-f_{j,i}(o_{\tau,j})\bigg)\right].
\end{align*}

Then, we show how the second term can be bounded with misspecification error $\epsilon$: 
\begin{align*}
    &\left|\mathbb{E}\left[\frac{1}{\sum_{j\in \mathcal{N}(i)} n_t(j) }\sum_{j\in \mathcal{N}(i)
    }\sum_{\tau=1}^t \mathbb{I}_\tau(j)\bigg(\hat{f}_{j,i}(o_{\tau,j})-f_{j,i}(o_{\tau,j})\bigg)\right]\right|\\
    & \hspace{2cm}\le \mathbb{E}\left[\frac{1}{\sum_{j\in \mathcal{N}(i)} n_t(j) }\sum_{j\in \mathcal{N}(i)}\sum_{\tau=1}^t \mathbb{I}_\tau(j)\bigg(\left|\hat{f}_{j,i}(o_{\tau,j})-f_{j,i}(o_{\tau,j})\right|\bigg)\right]\\
    & \hspace{2cm}\le \epsilon \cdot\mathbb{E}\left[\frac{1}{\sum_{j\in \mathcal{N}(i)} n_t(j) }\sum_{j\in \mathcal{N}(i)}\sum_{\tau=1}^t \mathbb{I}_t(j)\right]= \epsilon,
\end{align*}

therefore 
\begin{align*}
    \left|\mathbb{E}[\hat{r}_{t,i}]-\mathbb{E}[r_{t,i}]\right|\le \epsilon.
\end{align*}

Finally, we employ Azuma-Hoeffding inequality to bound the difference between the realization of $\hat{r}_{t,i}$ and its expected value. 
To do so, we observe that $-1\le 0-\epsilon\le \hat{f}_{j,i}(o_{t,j})\le 1+\epsilon\le 2$.
Then, with probability at least $1-\delta$

\begin{align*}
   \left|\hat{r}_{t,i}-\mathbb{E}[\hat{r}_{t,i}]\right|\le 6\sqrt{\frac{\ln(\delta^{-1})}{\sum_{j\in \mathcal{N}(i)}n_t(j)}}, 
\end{align*}
which concludes the proof. 
\end{proof}

Additionally, we can take the union bound, and rescaling $\delta$ to have a result that holds simultaneously for all $t\in [T]$, $i\in \mathcal{M}.$

\begin{lemma}
\label{lemma: uc est}
    Let $\delta\in (0,1). $ With probability at least $1-\delta$
    \begin{align*}
        \left|\hat{r}_{t,i}-\mathbb{E}[r_{t,i}]\right|\le 6\sqrt{\frac{\ln(TM\delta^{-1})}{\sum_{j\in \mathcal{N}(i)}n_t(j)}}+\epsilon \quad \forall i\in \mathcal{M}, \forall t\in [T]
    \end{align*}
\end{lemma}

Then, we can employ standard UCB-N analysis (\citet{lykouris2020feedback}) to bound the regret. Naturally, the regret will depend on the graph property.  
Therefore, we can prove the following result.

\theoremGraph*

\begin{proof}
Let $i^*\in \arg \max_{i\in \mathcal{M}} \mathbb{E}[r_{t,i}]$. Let $\mathcal{H}_t$ be the history up to round $t$, for all $t\in [T]$.
By \Cref{lemma: uc est}, with probability at least $1-\delta$
\begin{align*}
    \sum_{t=1}^T \left(\mathbb{E}[r_{t,i^*}]-\mathbb{E}[r_{t,i_t}|\mathcal{H}_{t-1}]\right)& = \sum_{t=1}^T \left(\mathbb{E}[r_{t,i^*}]-\overline{r}_{t,i^*}\right)\\
    & +\sum_{t=1}^T \left(\overline{r}_{t,i^*}-\overline{r}_{t,i_t}\right)+ \sum_{t=1}^T \left(\overline{r}_{t,i_t}-\mathbb{E}[r_{t,i_t}|\mathcal{H}_{t-1}]\right)\\
    & \le \sum_{t=1}^T \left(\overline{r}_{t,i_t}-\mathbb{E}[r_{t,i_t}|\mathcal{H}_{t-1}]\right)\\
    & \le 2\sum_{t=1}^T \left(6\sqrt{\frac{\ln(TM\delta^{-1})}{\sum_{j\in \mathcal{N}(i_t)}n_t(j)}}+\epsilon\right),
\end{align*}
where we used the fact that $i_t\in \arg\max_{i\in \mathcal{M}}\overline{r}_{t,i}$, and that with probability at least $1-\delta$ $\mathbb{E}[r_{t,i}]-\overline{r}_{t,i}\le 0$ for all $t\in [T]$, for all $i\in \mathcal{M}$. 
Therefore, aside from the addittive term $\mathcal{O}(T\cdot \epsilon)$, specific of our setting, the confidence intervals shrinking is equivalent to the one in the algorithm UCB-N proposed by \cite{lykouris2020feedback}, and following the same analysis we have

\begin{align*}
    R_T \le \mathcal{O}\left(\ln(TM\delta^{-1})\sqrt{\alpha(\mathcal{R})T} + 2\epsilon \,T\right).
\end{align*}



\end{proof}

\end{document}